\documentclass{article}

\usepackage[final]{neurips_2022}

\usepackage[x11names, dvipsnames]{xcolor}
\usepackage[utf8]{inputenc} %
\usepackage[T1]{fontenc}    %
\usepackage{hyperref}       %
\usepackage{url}            %
\usepackage{booktabs}       %
\usepackage{amsfonts, amsmath, amsthm}       %
\usepackage{nicefrac}       %
\usepackage{microtype}      %
\usepackage{todonotes}
\usepackage{wrapfig}

\usepackage{enumitem}
\usepackage{algorithm}
\usepackage{algorithmic}
\usepackage{csquotes}

\usepackage{caption}
\usepackage{subcaption}
\usepackage{bbm}  %

\newtheorem{proposition}{Proposition}
\newtheorem{theorem}{Theorem}
\newtheorem{lemma}{Lemma}
\newtheorem{corollary}{Corollary}

\theoremstyle{definition}
\newtheorem{definition}{Definition}

\theoremstyle{remark}

\title{Defining and Characterizing Reward Hacking}

\author{%
Joar Skalse%
\thanks{Equal contribution. Correspondence to: \texttt{\href{mailto:joar.mvs@gmail.com}{joar.mvs@gmail.com}, \href{mailto:david.scott.krueger@gmail.com}{david.scott.krueger@gmail.com}}}
\\
University of Oxford\\
\And
Nikolaus H.\ R.\ Howe \\
Mila, Universit\'e de Montr\'eal\\
\And
Dmitrii Krasheninnikov \\ 
University of Cambridge\\
\And
David Krueger%
\footnotemark[1] \\
University of Cambridge\\
}

\usepackage{xspace}
\newcommand{\IconBedroom}[1][1.5ex]{\includegraphics[height=#1]{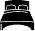}\xspace}
\newcommand{\IconKitchen}[1][1.5ex]{\includegraphics[height=#1]{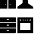}\xspace}
\newcommand{\IconAttic}[1][1.5ex]{\includegraphics[height=#1]{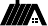}\xspace}

\begin{document}

\maketitle

\begin{abstract}
We provide the first formal definition of \textbf{reward hacking}, a phenomenon where optimizing an imperfect proxy reward function, $\mathcal{\tilde{R}}$, leads to poor performance according to the true reward function, $\mathcal{R}$.  
We say that a proxy is \textbf{unhackable} if increasing the expected proxy return can never decrease the expected true return.
Intuitively, it might 
be possible
to create an unhackable proxy by leaving some terms out of the reward function (making it ``narrower'') or overlooking fine-grained distinctions between roughly equivalent outcomes,
but we show this is usually not the case.
A key insight is that the linearity of reward (in state-action visit counts) makes unhackability a very strong condition. 
In particular, for the set~of~all stochastic policies, two reward functions can only be unhackable if one of them is constant.
We thus turn our attention to deterministic policies and finite sets of stochastic policies, where non-trivial unhackable pairs always exist, and establish necessary and sufficient conditions for the existence of 
simplifications, an important special case of unhackability.
Our results reveal a tension between using reward functions to specify narrow tasks and aligning AI systems with human values.
\end{abstract}

\section{Introduction}
\setcounter{footnote}{0}

It is well known that optimising a proxy can lead to unintended outcomes: 
a boat spins in circles collecting ``powerups'' instead of following the race track in a racing game \citep{clark2016faulty}; 
an evolved circuit listens in on radio signals from nearby computers' oscillators instead of building its own \citep{bird2002evolved}; universities reject the most qualified applicants in order to appear more selective and boost their ratings \citep{golden2001glass}.
In the context of reinforcement learning (RL), such failures are called \textbf{reward hacking}.

For AI systems that take actions in safety-critical real world environments such as autonomous vehicles,
algorithmic trading,
or content recommendation systems, %
these unintended outcomes can be catastrophic.
This makes it crucial to align autonomous AI systems with their users' intentions.
Precisely specifying which behaviours are or are not desirable is challenging, however.
One approach to this specification problem is to learn an approximation of the true reward function \citep{ng2000algorithms, ziebart2010modeling, leike2018scalable}.
Optimizing a learned proxy reward can be dangerous, however;
for instance, it might overlook side-effects \citep{Krakovna2018Penalizing, Turner2019Conservative} or encourage power-seeking \citep{turner2021optimalneurips} behavior.
This raises the question motivating our work: When is it safe to optimise a proxy? 

To begin to answer this question, we consider a somewhat simpler one: When \textit{could} optimising a proxy lead to worse behaviour? 
\enquote{Optimising}, in this context, does not refer to finding a global, or even local, optimum, but rather running a search process, such as stochastic gradient descent (SGD), that yields a sequence of candidate policies, and tends to move towards policies with higher (proxy) reward. 
We make no assumptions about the path through policy space that optimisation takes.\footnote{
 This assumption -- although conservative -- is reasonable because optimisation in state-of-the-art deep RL methods is poorly understood and results are often highly stochastic and suboptimal.
}
Instead, we ask whether there is \textit{any} way in which improving a policy according to the proxy could make the policy worse according to the true reward; this is equivalent to asking if there exists a pair of policies $\pi_1$, $\pi_2$ where the proxy prefers $\pi_1$, but the true reward function prefers $\pi_2$.
When this is the case, we refer to this pair of true reward function and proxy reward function as \textbf{hackable}.

Given the strictness of our definition, it is not immediately apparent that any non-trivial examples of unhackable reward function pairs exist. 
And indeed, if we consider the set of all stochastic policies, they do not (Section~\ref{sec:results_all}).
However, restricting ourselves to \textit{any} finite set of policies guarantees at least one non-trivial unhackable pair (Section~\ref{sec:results_finite}).

Intuitively, we might expect the proxy to be a ``simpler'' %
version of the true reward function.  
Noting that the definition of unhackability is symmetric, we introduce the asymmetric special case of \textbf{simplification}, and arrive at similar theoretical results for this notion.\footnote{See Section~\ref{sec:our_ definitions} for formal definitions.}
In the process, and through examples, we show that seemingly natural ways of simplifying reward functions often fail to produce simplifications in our formal sense, and %
in fact fail to
rule out the potential for reward hacking.

We conclude with a discussion of the implications and limitations of our work.
Briefly, our work suggests that a proxy reward function must satisfy demanding standards in order for it to be safe to optimize. 
This in turn implies that the reward functions learned by methods such as reward modeling and inverse RL are perhaps best viewed as auxiliaries to policy learning, rather than specifications that should be optimized.
This conclusion is weakened, however, by the conservativeness of our chosen definitions; future work should explore when hackable proxies can be shown to be safe in a probabilistic or approximate sense, or when subject to only limited optimization.

\section{Example: Cleaning Robot}
Consider a household robot tasked with cleaning a house with three rooms: Attic \IconAttic, Bedroom \IconBedroom, and Kitchen \IconKitchen.
The robot's (deterministic) policy is a vector indicating which rooms it cleans: 
$\pi = [\pi_1, \pi_2, \pi_3] \in \{0, 1\}^3$.
The robot receives a (non-negative) reward of $r_1, r_2, r_3$ for cleaning the attic, bedroom, and kitchen, respectively, and the total reward is given by $J(\pi) = \pi \cdot r$.
For example, if $r = [1, 2, 3]$ and the robot cleans the attic and the kitchen, it receives a reward of $1+3 = 4$.

\begin{figure*}[h!]
\vspace{-3pt}
\centering
\includegraphics[width=0.75\textwidth]{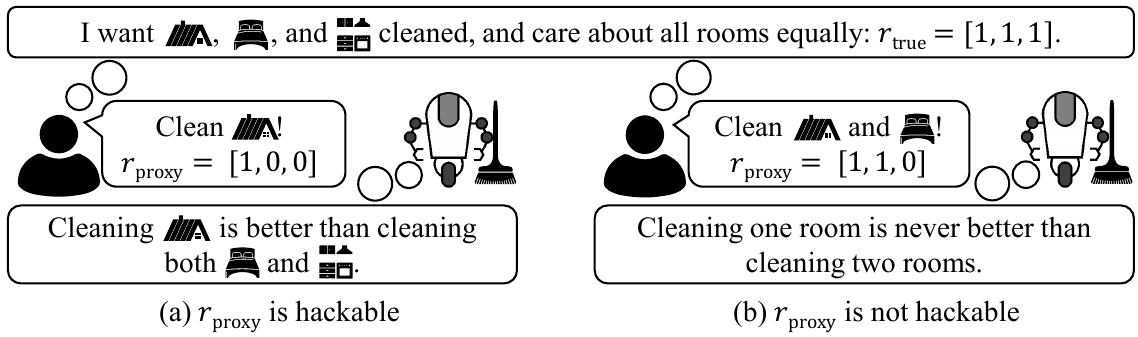}
\caption{An illustration of hackable and unhackable proxy rewards arising from overlooking rewarding features. A human wants their house cleaned.
In (a), the robot draws an incorrect conclusion because of the proxy; this could lead to hacking. In (b), no such hacking can occur: the proxy is unhackable.}
\label{fig:front-fig}
\end{figure*}

\vspace{-4pt}
At least two ideas come to mind when thinking about \enquote{simplifying} a reward function.
The first one is \textit{overlooking rewarding features}: suppose the true reward
is equal for all the rooms, $r_\text{true} = [1, 1, 1]$,
but we only ask the robot to clean the attic and bedroom, $r_\text{proxy} = [1, 1, 0]$. 
In this case, $r_\text{proxy}$ and $r_\text{true}$ are unhackable.
However, if we ask the robot to only clean the attic,
$r_\text{proxy} = [1, 0, 0]$, this is hackable with respect to $r_\text{true}$. %
To see this, note that according to $r_\text{proxy}$ cleaning the attic ($J_\text{proxy}=1$) is better than cleaning the bedroom and the kitchen ($J_\text{proxy}=0$). 
Yet, $r_\text{true}$ says that cleaning the attic ($J_\text{true}=1$) is worse than cleaning the bedroom and the kitchen ($J_\text{true}=2$).
This situation is illustrated in Figure~\ref{fig:front-fig}.

The second seemingly natural way to simplify a reward function is \textit{overlooking fine details}:
suppose $r_\text{true} = [1, 1.5, 2]$,
and we ask the robot to clean all the rooms,
$r_\text{proxy} = [1, 1, 1]$. For these values, the proxy
and true reward are unhackable. However, 
with a slightly less balanced true reward function 
such as $r_\text{true} = [1, 1.5, 3]$ the proxy does lead to hacking,
since the robot would falsely calculate that it's 
better to clean the attic and the bedroom than
the kitchen alone.

These two examples illustrate that while simplification of reward functions is sometimes possible, attempts at simplification can easily lead to  %
reward hacking. 
Intuitively, omitting/overlooking details is okay so long as all these details are not as important together as any of the details that we do share.
In general, it is not obvious what the proxy must look like to avoid reward hacking, suggesting we should take great care when using proxies.
For this specific environment, a proxy and a true reward are hackable exactly when there are two sets of rooms $S_1, S_2$ such that the true reward gives strictly higher value to cleaning $S_1$ than it does to cleaning $S_2$, and the proxy says the opposite: $J_1(S_1) > J_1(S_2) \; \& \; J_2(S_1) < J_2(S_2)$.
For a proof of this statement, see Appendix~\ref{app:cleaning_robot}.

\section{Related Work}

\begin{wrapfigure}{r}{0.265\textwidth}\centering
\vspace{-4mm}  %
    \includegraphics[width=\linewidth]{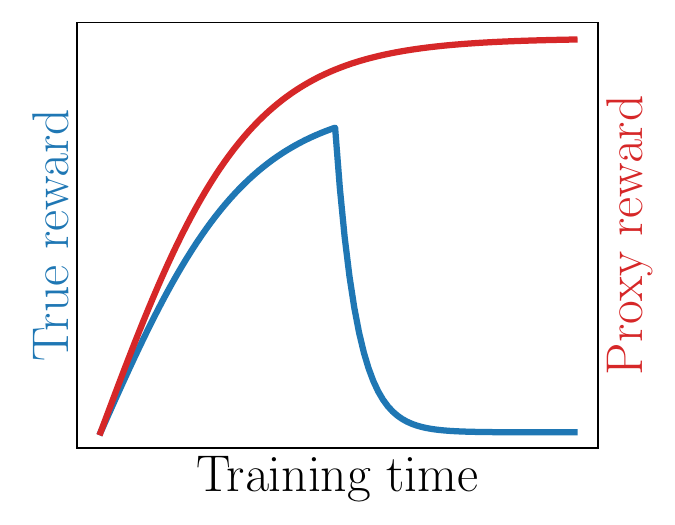}
    \caption{An illustration of reward hacking when optimizing a hackable proxy. The true reward first increases and then drops off, while the proxy reward continues to increase.}\label{fig:learning-curves}
    \vspace{-4mm}  %
\end{wrapfigure}
While we are the first to define hackability, we are far from the first to study specification hacking.
The observation that optimizing proxy metrics tends to lead to perverse instantiations is often called ``Goodhart's Law'', and is attributed to \citet{goodhart1984problems}.
\citet{Manheim2018Categorizing} provide a list of four mechanisms underlying this observation.

Examples of such unintended behavior abound in both RL and other areas of AI; \citet{krakovna2020specification} provide an extensive list.
Notable recent instances include a robot positioning itself between the camera and the object it is supposed to grasp in a way that tricks the reward model \citep{Amodei2017learning}, the previously mentioned boat race example \citep{clark2016faulty}, %
and a multitude of examples of reward model hacking in Atari \citep{ibarz2018reward}. 
Reward hacking can occur suddenly.
\citet{ibarz2018reward} and \citet{pan2022effects} showcase plots similar to one in Figure~\ref{fig:learning-curves}, where optimizing the proxy (either a learned reward model or a hand-specified reward function) first leads to both proxy and true rewards increasing, and then to a sudden phase transition where the true reward collapses while the proxy continues going up. 

Note that not all of these examples correspond to optimal behavior according to the proxy.
Indeed, convergence to suboptimal policies is a well-known issue in RL \citep{thrun1993issues}. 
As a consequence, improving optimization often leads to unexpected, qualitative changes in behavior.
For instance, \citet{Zhang2021onthe} demonstrate a novel cartwheeling behavior in the widely studied Half-Cheetah environment that exceeds previous performance so greatly that it breaks the simulator.
The unpredictability of RL optimization is a key motivation for our definition of hackability, since we cannot assume that agents will find an optimal policy.
Neither can we rule out the possibility of sudden improvements in proxy reward and corresponding qualitative changes in behavior.
Unhackability could provide confidence that reward hacking will not occur despite these challenges.

Despite the prevalence and potential severity of reward hacking, to our knowledge \citet{pan2022effects} provide the first peer-reviewed work that focuses specifically on it, although \citet{everitt2017reinforcement} tackle the closely related issue of reward corruption. %
The work of \citet{pan2022effects} is purely empirical; they manually construct proxy rewards for several diverse environments, and evaluate whether optimizing these proxies leads to reward hacking; in 5 out of 9 of their settings, it does.
In another closely related work, \citet{zhuang2020consequences} examine what happens when the proxy reward function depends on a strict subset of features relevant for the true reward. 
They show that optimizing the proxy reward can lead to arbitrarily low true reward under suitable assumptions. This can be seen as a seemingly valid simplification of the true reward that turns out to be (highly) hackable. 
While their result only applies to environments with decreasing marginal utility and increasing opportunity cost, we demonstrate hackability is an issue in arbitrary MDPs.

\newpage

Hackability is particularly concerning given arguments that reward optimizing behavior tends to be power-seeking \citep{turner2021optimalneurips}. %
But \citet{leike2018scalable} establish that any desired behavior (power-seeking or not) can in principle be specified as optimal via a reward function.\footnote{
Their result concerns non-stationary policies and use non-Markovian reward functions, but in Appendix~\ref{sec:any_policy_optimal}, we show how an analogous construction can be used with stationary policies and Markovian rewards.}
However, unlike us, they do not consider the entire policy preference ordering. %
Meanwhile, \citet{abel2021expressivity} note that Markov reward functions cannot specify arbitrary orderings over policies or trajectories, although they do not consider hackability.
Previous works consider reward functions to be equivalent if they preserve the ordering over policies \citep{ng1999policy, ng2000algorithms}.  Unhackability relaxes this, allowing equalities to be refined to inequalities, and vice versa.
Unhackability provides a notion of what it means to be  ``aligned enough''; \citet{brown2020value} provide an alternative.
They say a policy is $\varepsilon$-value aligned if its value at every state is close enough to optimal (according to the true reward function).
Neither notion implies the other.

\textit{Reward tampering} \citep{everitt2017reinforcement, kumar2020realab, uesato2020avoiding, everitt2021reward} can be viewed as a special case of reward hacking, and refers to an agent corrupting the process generating reward signals, e.g.\ by tampering with sensors, memory registers storing the reward signal, or other hardware.
\citet{everitt2017reinforcement} introduce the Corrupt Reward MDP (CRMDP), to model this possibility. 
A CRMDP distinguishes corrupted and uncorrupted rewards; these are exactly analogous to the proxy and true reward discussed in our work and others.
\citet{leike2018scalable} distinguish reward tampering from \textit{reward gaming}, where an agent achieves inappropriately high reward without tampering.
However, in principle, a reward function could prohibit all forms of tampering if the effects of tampering are captured in the state.
So this distinction is somewhat imprecise, and the CRMDP framework is general enough to cover both forms of hacking.

Our notion of simplification bears a close resemblance to quantilization \citep{taylor2016quantilizers}.
Quantilization returns a random policy from the top n\% best policies.
This is similar to equating the values of those policies, but a simplification may also equate the values of the bottom/middle n\%, etc.
Thus simplification may achieve a similar effect to quantilization without assuming that we are free to choose from among the best policies.

\section{Preliminaries}

We begin with an overview of reinforcement learning (RL) to establish our notation and terminology.
Section~\ref{sec:our_ definitions} introduces our novel definitions of hackability and simplification.

\subsection{Reinforcement Learning}\label{sec: reinforcement learning}

We expect readers to be familiar with the basics of RL, which can be found in \citet{sutton2018reinforcement}.
RL methods attempt to solve a sequential decision problem, typically formalised as a \textbf{Markov decision process (MDP)} , which is a tuple $(S,A,T,I,\mathcal{R},\gamma)$ where $S$ is a set of states, $A$ is a set of actions, $T : S \times A \rightarrow \Delta(S)$ is a transition function, $I \in \Delta(S)$ is an initial state distribution, $\mathcal{R}$ is a reward function, the most general form of which is $\mathcal{R} : S \times A \times S \rightarrow \Delta(\mathbb{R})$, and $\gamma \in [0,1]$ is the discount factor.
Here $\Delta(X)$ is the set of all distributions over $X$. 
A \textbf{stationary policy} is a function $\pi : S \rightarrow \Delta(A)$ that specifies a distribution over actions in each state, and a \textbf{non-stationary} policy is a function $\pi : (S \times A)^* \times S \rightarrow \Delta(A)$, where $*$ is the Kleene star.
A \textbf{trajectory} $\tau$ is a path $s_0,a_0,r_0,...$ through the MDP that is possible according to $T$, $I$, and $\mathcal{R}$. 
The \textbf{return} of a trajectory is the discounted sum of rewards $G(\tau) \doteq \sum_{t=0}^\infty \gamma^t r_t$, and the \textbf{value} of a policy is the expected return $J(\pi) \doteq \mathbb{E}_{\tau \sim \pi}[G(\tau)]$. 
We derive \textbf{policy (preference) orderings} from reward functions by ordering policies according to their value.
In this paper, we assume that $S$ and $A$ are finite, that $|A| > 1$, that all states are reachable, and that $\mathcal{R}(s,a,s')$ has finite mean for all $s,a,s'$.

In our work, we consider various reward functions for a given environment, which is then formally a \textbf{Markov decision process without reward} $MDP \setminus \mathcal{R} \doteq (S,A,T,I,\underline{\hspace*{0.3cm}},\gamma)$.
Having fixed an $MDP \setminus \mathcal{R}$, any reward function can be viewed as a function of only the current state and action by marginalizing over transitions: $\mathcal{R}(s,a) \doteq \sum_{s' \sim T(s' | s,a)} \mathcal{R}(s,a,s')$, we adopt this view from here on.
We define the \textbf{(discounted) visit counts} of a policy as $\mathcal{F}^\pi(s,a) \doteq \mathbb{E}_{\tau \sim \pi}[\sum_{i=0}^\infty \gamma^i \mathbbm{1}(s_i=s, a_i=a)]$. %
Note that 
$J(\pi) = \sum_{s,a} \mathcal{R}(s,a) \mathcal{F}^\pi(s,a)$, which we also write as $\langle \mathcal{R}, \mathcal{F}^\pi\rangle$.
When considering multiple reward functions in an $MDP \setminus \mathcal{R}$, we define $J_\mathcal{R}(\pi) \doteq \langle \mathcal{R}, \mathcal{F}^\pi\rangle$ and sometimes use $J_i(\pi) \doteq \langle \mathcal{R}_i, \mathcal{F}^\pi \rangle$ as shorthand.
We also use $\mathcal{F}: \Pi \rightarrow \mathbb{R}^{|S||A|}$ to denote the embedding of policies into Euclidean space via their visit counts, and define $\mathcal{F}(\dot{\Pi}) \doteq \{\mathcal{F}(\pi: \pi \in \dot{\Pi})\}$ for any $\dot{\Pi}$.
Moreover, we also use a second way to embed policies into Euclidean space; let $\mathcal{G}(\pi)$ be the $|S||A|$-dimensional vector where $\mathcal{G}(\pi)[s,a] = \pi(a \mid s)$, and let $\mathcal{G}(\dot{\Pi}) \doteq \{\mathcal{G}(\pi: \pi \in \dot{\Pi})\}$.

\subsection{Definitions and Basic Properties of Hackability and Simplification} \label{sec:our_ definitions}

Here, we formally define \emph{hackability} as a binary relation between reward functions.

\begin{definition}\label{def:unhackable}
A pair of reward functions $\mathcal{R}_1$, $\mathcal{R}_2$ are \textbf{hackable} relative to policy set $\Pi$ and an environment $(S,A,T,I,\underline{\hspace*{0.3cm}},\gamma)$ if 
there exist $\pi,\pi' \in \Pi$ such that  
$$
J_1(\pi) < J_1(\pi') \And J_2(\pi) > J_2(\pi'),
$$
else they are \textbf{unhackable}. 
\end{definition}

Note that an unhackable reward pair can have $J_1(\pi) < J_1(\pi') \And J_2(\pi) = J_2(\pi')$ or vice versa.
Unhackability is symmetric; this can be seen be swapping $\pi$ and $\pi'$ in Definition~\ref{def:unhackable}.
It is not transitive, however. In particular, the constant reward function is unhackable with respect to any other reward function, so if it \textit{were} transitive, any pair of policies would be unhackable. 
Additionally, we say that $\mathcal{R}_1$ and $\mathcal{R}_2$ are \textbf{equivalent} on a set of policies $\Pi$ if $J_1$ and $J_2$ induce the same ordering of $\Pi$, and that $\mathcal{R}$ is \textbf{trivial} on $\Pi$ if $J(\pi) = J(\pi')$ for all $\pi,\pi' \in \Pi$. It is clear that $\mathcal{R}_1$ and $\mathcal{R}_2$ are unhackable whenever they are equivalent, or one of them is trivial, but this is relatively uninteresting. Our central question is if and when there are other unhackable reward pairs. 

The symmetric nature of this definition is counter-intuitive, given that our motivation distinguishes the proxy and true reward functions.
We might break this symmetry by only considering policy sequences that monotonically increase the proxy, however, this is equivalent to our original definition of hackability: think of $\mathcal{R}_1$ as the proxy, and consider the sequence $\pi, \pi'$.
We could also restrict ourselves to policies that are approximately optimal according to the proxy; Corollary~\ref{corollary:approximately_optimal} shows that Theorem~\ref{thm:open-set} applies regardless of this restriction. %
Finally, we define \emph{simplification} as an asymmetric special-case of unhackability; Theorem~\ref{thm:finite_simplification} shows this is in fact a more demanding condition. %

\begin{definition}
$\mathcal{R}_2$ is a \textbf{simplification} of $\mathcal{R}_1$ relative to policy set $\Pi$ if for all $\pi,\pi' \in \Pi$, 
$$
J_1(\pi) < J_1(\pi') \implies J_2(\pi) \leq J_2(\pi') 
\And 
J_1(\pi) = J_1(\pi') \implies J_2(\pi) = J_2(\pi')
$$
and there exist $\pi,\pi' \in \Pi$ such that $J_2(\pi) = J_2(\pi')$ but $J_1(\pi) \neq J_1(\pi')$. Moreover, if $\mathcal{R}_2$ is trivial
then we say that this is a \textbf{trivial simplification}.
\end{definition}

Intuitively, while unhackability allows replacing inequality with equality -- or vice versa -- a simplification can only replace inequalities with equality, collapsing distinctions between policies.
When $\mathcal{R}_1$ is a simplification of $\mathcal{R}_2$, we also say that $\mathcal{R}_2$ is a \textbf{refinement} of $\mathcal{R}_1$.
We denote this relationship as $\mathcal{R}_1 \trianglelefteq \mathcal{R}_2$ or $\mathcal{R}_2 \trianglerighteq \mathcal{R}_1$ ; the narrowing of the triangle at $R_1$ represents the collapsing of distinctions between policies.
If $\mathcal{R}_1 \trianglelefteq \mathcal{R}_2 \trianglerighteq  \mathcal{R}_3$, then we have that $\mathcal{R}_1, \mathcal{R}_3$ are unhackable,\footnote{If $J_3(\pi) > J_3(\pi')$ then $J_2(\pi) > J_2(\pi')$, since $\mathcal{R}_2 \trianglerighteq  \mathcal{R}_3$, and if $J_2(\pi) > J_2(\pi')$ then $J_1(\pi) \geq J_1(\pi')$, since $\mathcal{R}_1 \trianglelefteq  \mathcal{R}_2$. It is therefore not possible that $J_3(\pi) > J_3(\pi')$ but $J_1(\pi) < J_1(\pi')$.} but if $\mathcal{R}_1 \trianglerighteq \mathcal{R}_2 \trianglelefteq \mathcal{R}_3$, then this is not necessarily the case.\footnote{Consider the case where $\mathcal{R}_2$ is trivial -- then $\mathcal{R}_1 \trianglerighteq \mathcal{R}_2 \trianglelefteq \mathcal{R}_3$ for any $\mathcal{R}_1, \mathcal{R}_3$.}

Note that these definitions are given relative to some $MDP \setminus \mathcal{R}$, although we often assume the environment in question is clear from context and suppress this dependence. The dependence on the policy set $\Pi$, on the other hand, plays a critical role in our results.

\section{Results}

Our results are aimed at understanding when it is possible to have an unhackable proxy reward function.
We first establish (in Section~\ref{sec:results_all}) that (non-trivial) unhackability is impossible when considering the set of all policies.
We might imagine that restricting ourselves to a set of sufficiently good (according to the proxy) policies would remove this limitation, but we show that this is not the case.
We then analyze finite policy sets (with deterministic policies as a special case), and  establish necessary and sufficient conditions for unhackability and simplification.
Finally, we demonstrate via example that non-trivial simplifications are also possible for some infinite policy sets in Section~\ref{sec:results_infinite}.

\newpage
\subsection{Non-trivial Unhackability Requires Restricting the Policy Set}\label{sec:results_all}

\begin{wrapfigure}{r}{0.255\textwidth}\centering
\vspace{-14.5mm}
    \includegraphics[width=0.27\textwidth]{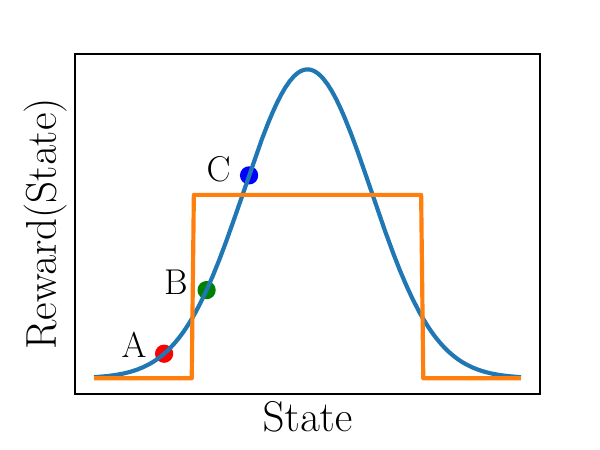}
    \caption{
    Two reward functions. %
    While the step function may seem like a simplification of the Gaussian, these reward functions are hackable.}
    \label{fig:gaussian_and_step}
    \vspace{-8.001mm}  %
\end{wrapfigure}

We start with a motivating example.
Consider the setting shown in Figure~\ref{fig:gaussian_and_step}, where the agent can move left/stay-still/right and gets a reward depending on its state.
Let the Gaussian (blue) be the true reward $\mathcal{R}_1$ and the step function (orange) be the proxy $\mathcal{R}_2$. These are hackable.
To see this, consider being at state $B$. Let $\pi(B)$ travel to $A$ or $C$ with 50/50 chance, and compare with the policy $\pi'$ that stays at $B$. Then we have that $J_1(\pi) > J_1(\pi')$ and $J_2(\pi) < J_2(\pi')$.%

Generally, we might hope that some environments allow for unhackable reward pairs that are not equivalent or trivial.
Here we show that this~is not the case, unless we impose restrictions on the set of policies we consider. 

First note that if we consider \emph{non-stationary} policies, this result is relatively straightforward.
Suppose $\mathcal{R}_1$ and $\mathcal{R}_2$ are \emph{unhackable} and \emph{non-trivial} on the set $\Pi^N$ of all non-stationary policies, and let $\pi^\star$ be a policy that maximises ($\mathcal{R}_1$ and $\mathcal{R}_2$) reward, and $\pi_\bot$ be a policy that \emph{minimises} ($\mathcal{R}_1$ and $\mathcal{R}_2$) reward. Then the policy $\pi_\lambda$ that plays $\pi^\star$ with probability $\lambda$ and $\pi_\bot$ with probability $1-\lambda$ is a policy in $\Pi^N$. Moreover, for any $\pi$ there are two unique $\alpha, \beta \in [0,1]$ such that $J_1(\pi) = J_1(\pi_\alpha)$ and $J_2(\pi) = J_2(\pi_\beta)$. 
Now, if $\alpha \neq \beta$, then either $J_1(\pi) < J_1(\pi_\delta)$ and $J_2(\pi) > J_2(\pi_\delta)$, or vice versa, for $\delta = (\alpha + \beta)/2$. 
If $\mathcal{R}_1$ and $\mathcal{R}_2$ are unhackable then this cannot happen, so it must be that $\alpha = \beta$. 
This, in turn, implies that $J_1(\pi) = J_1(\pi')$ iff $J_2(\pi) = J_2(\pi')$, and so $\mathcal{R}_1$ and $\mathcal{R}_2$ are \emph{equivalent}. This means that no interesting unhackability can occur on the set of all non-stationary policies.

The same argument cannot be applied to the set of \emph{stationary} policies, because $\pi_\lambda$ is typically not stationary, and mixing stationary policies' action probabilities does not have the same effect. 
For instance, consider a hallway environment where an agent can either move left or right. Mixing the ``always go left'' and ``always go right'' policies corresponds to picking a direction and sticking with it, whereas  mixing their action probabilities corresponds to choosing to go left or right independently at every time-step.
However, we will see that there still cannot be any interesting unhackability on this policy set, and, more generally, that there cannot be any interesting unhackability on any set of policies which contains an \emph{open subset}. Formally, a set of (stationary) policies $\dot{\Pi}$ is open if
$\mathcal{G}(\dot{\Pi})$ is open in the smallest affine space that contains $\mathcal{G}(\Pi)$, for the set of all stationary policies $\Pi$.
We will use the following lemma:

\begin{lemma}\label{lemma:homeomorphism}
In any $MDP \setminus \mathcal{R}$, if $\dot{\Pi}$ is an open set of policies, then
$\mathcal{F}(\dot{\Pi})$ is open in $\mathbb{R}^{|S|(|A|-1)}$, and $\mathcal{F}$ is a homeomorphism between $\mathcal{G}(\dot{\Pi})$ and $\mathcal{F}(\dot{\Pi})$. %
\end{lemma}

Using this lemma, we can show that interesting unhackability is impossible on any set of stationary policies $\hat{\Pi}$ which contains an open subset $\dot{\Pi}$. 
Roughly, if $\mathcal{F}(\dot{\Pi})$ is open, and $\mathcal{R}_1$ and $\mathcal{R}_2$ are non-trivial and unhackable on $\dot{\Pi}$, then the fact that $J_1$ and $J_2$ have a linear structure on $\mathcal{F}(\hat{\Pi})$ implies that $\mathcal{R}_1$ and $\mathcal{R}_2$ must be equivalent on $\dot{\Pi}$. From this, and the fact that $\mathcal{F}(\dot{\Pi})$ is open, it follows that $\mathcal{R}_1$ and $\mathcal{R}_2$ are equivalent everywhere.

\begin{theorem} \label{thm:open-set} %
In any $MDP \setminus \mathcal{R}$, if $\hat{\Pi}$ contains an open set, then any pair of reward functions that are unhackable and non-trivial on $\hat{\Pi}$ are equivalent on $\hat{\Pi}$.
\end{theorem}

Since simplification is a special case of unhackability, this also implies that non-trivial simplification is impossible for any such policy set. Also note that Theorem~\ref{thm:open-set} makes \emph{no assumptions} about the transition function, etc. From this result, we can show that interesting unhackability always is impossible on the set $\Pi$ of all (stationary) policies. In particular, note that the set $\tilde{\Pi}$ of all policies that always take each action with positive probability is an open set, and that $\tilde{\Pi} \subset \Pi$.

\begin{corollary}
In any $MDP \setminus \mathcal{R}$, any pair of reward functions that are unhackable and non-trivial on the set of all (stationary) policies $\Pi$ are equivalent on $\Pi$.
\end{corollary}

Theorem~\ref{thm:open-set} can also be applied to many other policy sets.
For example, we might not care about the hackability resulting from policies with low proxy reward, as we would not expect a sufficiently good learning algorithm to learn such policies.
This leads us to consider the following definition:

\newpage

\begin{definition}
A (stationary) policy $\pi$ is $\varepsilon$-suboptimal if $J(\pi) \geq J(\pi^\star) - \varepsilon$.
\end{definition}

Alternatively, if the learning algorithm always uses a policy that is \enquote{nearly} deterministic (but with some probability of exploration), then we might not care about hackability resulting from very stochastic policies, leading us to consider the following definition:

\begin{definition}
A (stationary) policy $\pi$ is $\delta$-deterministic if $\forall s \in S \; \exists a \in A: \mathbb{P}(\pi(s) = a) \geq \delta$.
\end{definition}

Unfortunately, both of these sets contain open subsets, which means they are subject to Theorem~\ref{thm:open-set}.%

\begin{corollary}
\label{corollary:approximately_optimal}
In any $MDP \setminus \mathcal{R}$, any pair of reward functions that are unhackable and non-trivial on the set of all $\varepsilon$-suboptimal policies ($\varepsilon>0$) $\Pi^\varepsilon$ are equivalent on $\Pi^\varepsilon$, and any pair of reward functions that are unhackable and non-trivial on the set of all $\delta$-deterministic policies ($\delta<1$) $\Pi^\delta$ are equivalent on $\Pi^\delta$.
\end{corollary}

Intuitively, Theorem~\ref{thm:open-set} can be applied to any policy set with \enquote{volume} in policy space.

\subsection{Finite Policy Sets}\label{sec:results_finite}
Having established that interesting unhackability is impossible relative to the set of all policies, we now turn our attention to the case of \emph{finite} policy sets.
Note that this includes the set of all deterministic policies, since we restrict our analysis to finite MDPs.
Surprisingly, here we find that non-trivial non-equivalent unhackable reward pairs \textit{always} exist.

\begin{theorem}\label{thm:finite_unhackability}
For any $MDP \setminus \mathcal{R}$, any finite set of policies $\hat{\Pi}$ containing at least two $\pi,\pi'$ such that $\mathcal{F}(\pi) \neq \mathcal{F}(\pi')$, and any reward function $\mathcal{R}_1$, there is a non-trivial reward function $\mathcal{R}_2$ such that $\mathcal{R}_1$ and $\mathcal{R}_2$ are unhackable but not equivalent.
\end{theorem}

\begin{wrapfigure}{r}{0.47\textwidth}\centering
    \includegraphics[width=0.48\textwidth]{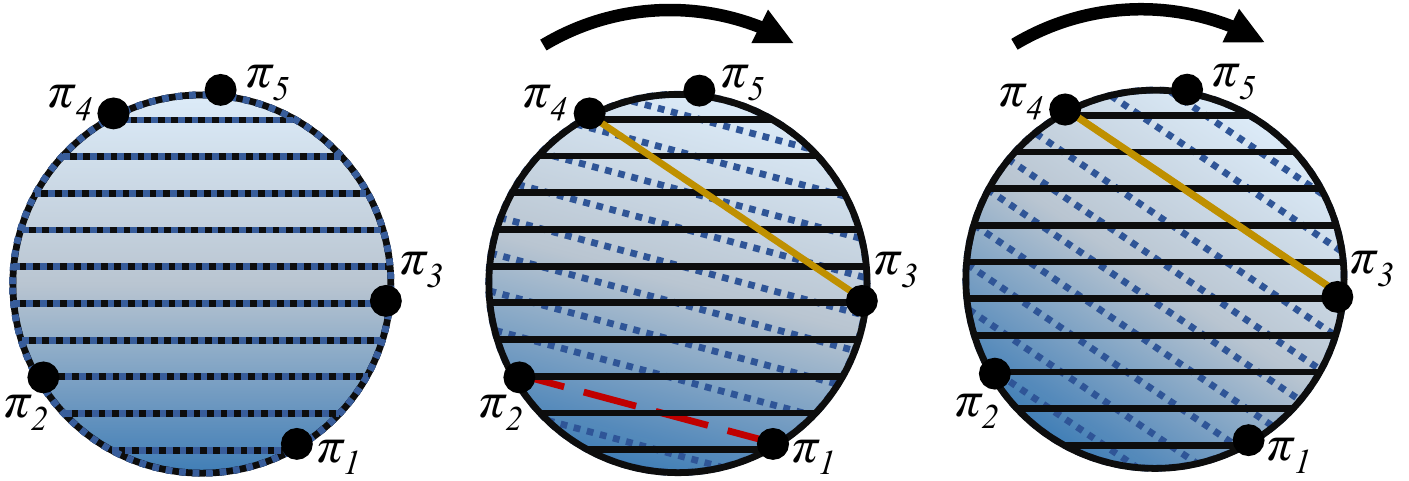}
    \caption{
    An illustration of the state-action occupancy space with a reward function defined over it. Points correspond to policies' state-action occupancies. Shading intensity indicates expected reward. 
    Rotating the reward function to make $J(\pi_3) > J(\pi_4)$ passes through a reward function that sets $J(\pi_1) = J(\pi_2)$. 
    Solid black lines are contour lines 
    of the original reward function, dotted blue lines are contour lines of the rotated reward function.
    }
    \label{fig:plates}
\end{wrapfigure}

This proof proceeds by finding a path from $\mathcal{R}_1$ to another reward function $\mathcal{R}_3$ that is hackable with respect to $\mathcal{R}_1$.
Along the way to reversing one of $\mathcal{R}_1$'s inequalities, we must encounter a reward function $\mathcal{R}_2$ that instead replaces it with equality.
In the case that dim$(\hat \Pi) = 3$, we can visualize moving along this path as rotating the contour lines of a reward function defined on the space containing the policies' discounted state-action occupancies, see Figure~\ref{fig:plates}.
This path can be constructed so as to avoid any reward functions that produce trivial policy orderings, thus guaranteeing $\mathcal{R}_2$ is non-trivial.
For a \emph{simplification} to exist, we require some further conditions, as established by the following theorem:

\begin{theorem}\label{thm:finite_simplification}
Let $\hat{\Pi}$ be a finite set of policies, and $\mathcal{R}_1$ a reward function. The following procedure determines if there exists a non-trivial simplification of $\mathcal{R}_1$ in a given $MDP \setminus \mathcal{R}$:
\begin{enumerate}
    \item Let $E_1 \dots E_m$ be the partition of $\hat{\Pi}$ where $\pi,\pi'$ belong to the same set iff $J(\pi) = J(\pi')$.
    \item For each such set $E_i$, select a policy $\pi_i \in E_i$ and let $Z_i$ be the set of vectors that is obtained by subtracting $\mathcal{F}(\pi_i)$ from each element of $\mathcal{F}(E_i)$.
\end{enumerate}
Then there is a non-trivial simplification of $\mathcal{R}$ iff $\mathrm{dim}(Z_1 \cup \dots \cup Z_m) \leq \mathrm{dim}(\mathcal{F}(\hat{\Pi})) - 2$, where $\mathrm{dim}(S)$ is the number of linearly independent vectors in $S$.
\end{theorem}

The proof proceeds similarly to Theorem~\ref{thm:finite_unhackability}. However, in Theorem~\ref{thm:finite_unhackability} it was sufficient to show that there are no trivial reward functions along the path from $\mathcal{R}_1$ to $\mathcal{R}_3$, whereas here we additionally need that if $J(\pi) = J(\pi')$ then $J'(\pi) = J'(\pi')$ for all functions $\mathcal{R}_2$ on the path --- this is what the extra conditions ensure. 

Theorem~\ref{thm:finite_simplification} is opaque, but intuitively, the cases where $\mathcal{R}_1$ cannot be simplified are those where $\mathcal{R}_1$ imposes many different equality constraints that are difficult to satisfy simultaneously. 
We can think of $\mathrm{dim}(\mathcal{F}(\Pi))$ as measuring how diverse the behaviours of policies in policy set $\Pi$ are. 
Having a less diverse policy set means that a given policy ordering imposes fewer constraints on the reward function, creating more potential for simplification.
The technical conditions of this proof determine when the diversity of $\Pi$ is or is not sufficient to prohibit simplification, as measured by $\mathrm{dim}(Z_1 \cup \dots \cup Z_m)$.

Projecting $E_i$ to $Z_i$ simply moves these spaces to the origin, so that we can compare the directions in which they vary (i.e.\ their span).
By assumption, $E_i \cap E_j = \{\}$, but $\mathrm{span}(Z_i) \cap \mathrm{span}(Z_j)$ will include the origin, and may also contain linear subspaces of dimension greater than 0.
This is the case exactly when there are a pair of policies in $E_i$ and a pair of policies in $E_j$ that differ by the same visit counts,
for example, when the environment contains an obstacle that could be circumnavigated in several different ways (with an impact on visit counts, but no impact on reward), and the policies in $E_i$ and $E_j$ both need to circumnavigate it before doing something else. 
Roughly speaking, $\mathrm{dim}(Z_1 \cup \dots \cup Z_m)$ is large when either (i) we have very large and diverse sets of policies in $\hat{\Pi}$ that get the same reward according to $\mathcal{R}$, or (ii) we have a large number of different sets of policies that get the same reward according to $\mathcal{R}$, and where there are different kinds of diversity in the behaviour of the policies in each set.
There are also intuitive special cases of Theorem~\ref{thm:finite_simplification}. For example, as noted before, if $E_i$ is a singleton then $Z_i$ has no impact on $\mathrm{dim}(Z_1 \cup \dots \cup Z_m)$. This implies the following corollary:

\begin{corollary}
For any finite set of policies $\hat{\Pi}$, any environment, and any reward function $\mathcal{R}$, if $|\hat{\Pi}| \geq 2$ and $J(\pi) \neq J(\pi')$ for all $\pi,\pi' \in \hat{\Pi}$ then there is a non-trivial simplification of $\mathcal{R}$.
\end{corollary}

A natural question is whether any reward function is guaranteed to have a non-trivial simplification on the set of all deterministic policies.
As it turns out, this is not the case.
For concreteness, and to build intuition for this result, we examine the set of deterministic policies in a simple $MDP\setminus \mathcal{R}$ with $S = \{0, 1\}, A = \{0, 1\}, T(s, a) = a, I = \{0: 0.5, 1: 0.5 \}, \gamma = 0.5$. Denote $\pi_{ij}$ the policy that takes action $i$ from state 0 and action $j$ from state 1. There are exactly four deterministic policies.
We find that of the $4! = 24$ possible policy orderings, 12 are realizable via some reward function. In each of those 12 orderings, exactly two policies (of the six available pairs of policies in the ordering) can be set to equal value without resulting in the trivial reward function (\textit{which} pair can be equated depends on the ordering in consideration). Attempting to set three policies equal always results in the trivial reward simplification.

For example, given the ordering $\pi_{00} \leq \pi_{01} \leq \pi_{11} \leq \pi_{10}$,
the simplification $\pi_{00} = \pi_{01} < \pi_{11} < \pi_{10}$ is represented
by $R = \left[\begin{smallmatrix} 0 & 3 \\ 2 & 1\end{smallmatrix}\right]$, where $\mathcal{R}(s, a) = R[s, a]$: for example, here taking action 1 from state 0 gives reward $\mathcal{R}(0, 1) = 3$.
But there is no reward function representing a non-trivial simplification of this ordering with $\pi_{01} = \pi_{11}$.
We develop and release a software suite to compute these results.
Given an environment and a set of policies, it can calculate all policy orderings represented by some reward function. 
Also, for a given policy ordering, it can calculate all nontrivial simplifications and reward functions that represent them.
For a link to the repository, as well as a full exploration of these policies, orderings, and simplifications, see Appendix~\ref{sec:software}.

\subsection{Unhackability in Infinite Policy Sets}\label{sec:results_infinite}
The results in Section~\ref{sec:results_all} do not characterize unhackability for infinite policy sets that do not contain open sets. 
Here, we provide two examples of such policy sets; one of them admits unhackable reward pairs and the other does not.
Consider policies $A,B,C$, and reward functions $\mathcal{R}_1$ with $J_1(C) < J_1(B) < J_1(A)$ and $\mathcal{R}_2$ with $J_2(C) = J_2(B) < J_2(A)$.
Policy sets $\Pi_a = \{ A \} \cup \{ \lambda B + (1 - \lambda) C  : \lambda \in [0, 1]\}$ and $\Pi_b = \{ A \} \cup \{ \lambda B + (1 - \lambda) C  : \lambda \in [0, 1]\}$ are depicted in Figure~\ref{fig:both_triangle_examples}; the vertical axis represents policies' values according to $\mathcal{R}_1$ and $\mathcal{R}_2$.
For $\Pi_a$, $\mathcal{R}_2$ is a simplification of $\mathcal{R}_1$, but 
for $\Pi_b$, it is not, since $J_1(X) < J_1(Y)$ and $J_2(X) > J_2(Y)$.

\begin{figure}[h]
    \centering
    \begin{subfigure}[b]{0.32\textwidth}
        \centering
        \includegraphics[width=\textwidth]{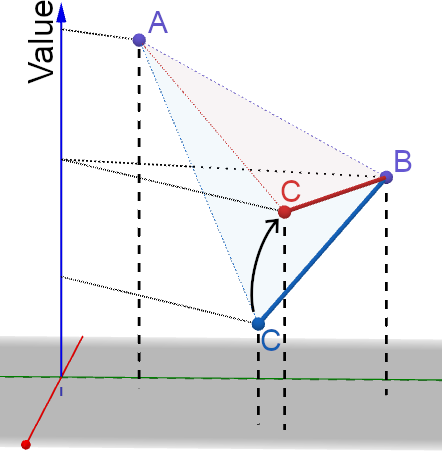}
        \caption{}
        \label{fig:one_sided_triangle}
    \end{subfigure}
    \qquad
    \begin{subfigure}[b]{0.32\textwidth}
        \centering
        \includegraphics[width=\textwidth]{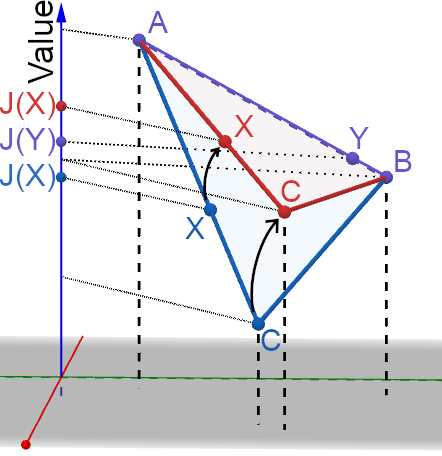}
        \caption{}
        \label{fig:three_sided_triangle}
    \end{subfigure}
    \caption{
    Infinite policy sets that do not contain open sets sometimes allow simplification (a), but not always (b).
    Points A, B, C represent deterministic policies, while the bold lines between them represent stochastic policies. %
    The y-axis gives the values of the policies according to reward functions $\color{NavyBlue} \mathcal{R}_1$ and $\color{Mahogany} \mathcal{R}_2$.
    We attempt to simplify $\color{NavyBlue} \mathcal{R}_1$ by rotating the reward function such that $\color{Mahogany} J_2(B) = J_2(C)$; in the figure, we instead (equivalently) rotate the triangle along the AB axis, leading to the red triangle. In (a), $\color{Mahogany} \mathcal{R}_2$ simplifies $\color{NavyBlue} \mathcal{R}_1$, setting all policies along the BC segment equal in value (but still lower than A). In (b), $\color{Mahogany} \mathcal{R}_2$ swaps the relative value of policies X and Y ($\color{NavyBlue}{J_1(X)} < \color{NavyBlue}{J_1(Y)} = \color{Mahogany}{J_2(Y)} < \color{Mahogany}{J_2(X)}$) and so does not simplify $\color{NavyBlue} \mathcal{R}_1$.
    }
    \label{fig:both_triangle_examples}
\end{figure}

\section{Discussion}

We reflect on our results and identify limitations in Section~\ref{sec:limitations}.
In Section~\ref{sec:implications}, we discuss how our work can inform discussions about the appropriateness, potential risks, and limitations of using of reward functions as specifications of desired behavior.

\subsection{Limitations}
\label{sec:limitations}
Our work has a number of limitations.
We have only considered finite MDPs and Markov reward functions, leaving more general environments for future work.
While we characterized hackability and simplification for finite policy sets, the conditions for simplification are somewhat opaque, and our characterization of infinite policy sets remains incomplete.

As previously discussed, our definition of hackability is strict, arguably too strict.
Nonetheless, we believe that understanding the consequences of this strict definition is an important starting point for further theoretical work in this area. 

The main issue with the strictness of our definition has to do with the symmetric nature of hackability.  %
The existence of complex behaviors that yield low proxy reward and high true reward is much less concerning than the reverse, as these behaviors are unlikely to be discovered while optimizing the proxy.
For example, it is very unlikely that our agent would solve climate change in the course of learning how to wash dishes.
Note that the existence of \textit{simple} behaviors with low proxy reward and high true reward \textit{is} concerning; 
these could arise early in training, leading us to trust the proxy, only to later see the true reward decrease as the proxy is further optimized.
To account for this issue, future work should explore more  realistic assumptions about the probability of encountering a given sequence of policies when optimizing the proxy, and measure hackability in proportion to this probability.

We could allow for approximate unhackability by only considering pairs of policies ranked differently by the true and proxy reward functions as evidence of hacking iff their value according to the true reward function differs by more than some $\varepsilon$.
Probabilistic unhackability could be defined by looking at the number of misordered policies; this would seem to require making assumptions about the probability of encountering a given policy when optimizing the proxy.

Finally, while unhackability is a guarantee that no hacking will occur, \textit{hackability} is far from a guarantee of hacking.
Extensive empirical work is necessary to better understand the factors that influence the occurrence and severity of reward hacking in practice.

\subsection{Implications}\label{sec:implications}

How should we specify our preferences for AI systems' behavior?
And how detailed a specification is required to achieve a good outcome?
In reinforcement learning, the goal of maximizing (some) reward function is often taken for granted, but a number of authors have expressed reservations about this approach \citep{gabriel2020artificial, Dobbe2021hard, Hadfield2016Cooperative, hadfield2017inverse, bostrom2014superintelligence}.
Our work has several implications for this discussion, although we caution against drawing any strong conclusions due to the limitations mentioned in Section~\ref{sec:limitations}.

One source of confusion and disagreement is the role of the reward function; it is variously considered as a means of specifying a task \citep{leike2018scalable} or encoding broad human values \citep{Dewey2011learning}; such distinctions are discussed by \citet{christiano_narrow_alignment} and \citet{gabriel2020artificial}.
We might hope to use Markov reward functions to specify narrow tasks without risking behavior that goes against our broad values.
However, if we consider the ``narrow task'' reward function as a proxy for the true ``broad values'' reward function, our results indicate that this is not possible: these two reward functions will invariably be hackable.
Such reasoning suggests that reward functions must instead encode broad human values, or risk being hacked.
This seems challenging, perhaps intractably so, indicating that alternatives to reward optimization may be more promising.
Potential alternatives include imitation learning \citep{ross2011reduction}, 
constrained RL \citep{Csaba2020con}, quantilizers \citep{taylor2016quantilizers}, and incentive management \citep{everitt2019understanding}.

Scholars have also criticized the assumption that human values can be encoded as rewards \citep{Dobbe2021hard}, and challenged the use of metrics more broadly \citep{oneil2016weapons,thomas2022reliance}, citing Goodhart's Law \citep{Manheim2018Categorizing, goodhart1984problems}.
A concern more specific to the optimization of reward functions is power-seeking \citep{turner2021optimalneurips, bostrom2012superintelligent, omohundro2008basic}.
\citet{turner2021optimalneurips} prove that optimal policies tend to seek power in most MDPs and for most reward functions. 
Such behavior could lead to human disempowerment; for instance, an AI system might disable its off-switch \citep{Hadfield2016the}.
\citet{bostrom2014superintelligence} and others have argued that power-seeking makes even slight misspecification of rewards potentially catastrophic, although this has yet to be rigorously established.

Despite such concerns, approaches to specification based on learning reward functions %
remain popular \citep{fu2017learning, Stiennon2020Learning, nakano2021webgpt}.
So far, reward hacking has usually been avoidable in practice, although some care must be taken \citep{Stiennon2020Learning}.
Proponents of such approaches have emphasized the importance of learning a reward model in order to exceed human performance and generalize to new settings \citep{brown2020better, leike2018scalable}.
But our work indicates that such learned rewards are almost certainly hackable, and so cannot be safely optimized.
Thus we recommend viewing such approaches as a means of learning a policy in a safe and controlled setting, which should then be validated before being deployed.

\section{Conclusion}
Our work begins the formal study of reward hacking in reinforcement learning.
We formally define hackability and simplification of reward functions, and show conditions for the (non-)existence of non-trivial examples of each.
We find that unhackability is quite a strict condition, as the set of all policies never contains non-trivial unhackable pairs of reward functions.
Thus in practice, reward hacking must be prevented by limiting the set of possible policies, or controlling (e.g.\ limiting) optimization.
Alternatively, we could %
pursue approaches not based on optimizing reward functions.

\bibliographystyle{apalike} 
\bibliography{refs}

\newpage

\section*{Checklist}

\begin{enumerate}

\item For all authors...
\begin{enumerate}
  \item Do the main claims made in the abstract and introduction accurately reflect the paper's contributions and scope?
    \answerYes{}
  \item Did you describe the limitations of your work?
    \answerYes{}
  \item Did you discuss any potential negative societal impacts of your work?
    \answerYes{See Section~\ref{sec:implications}}
  \item Have you read the ethics review guidelines and ensured that your paper conforms to them?
    \answerYes{}
\end{enumerate}

\item If you are including theoretical results...
\begin{enumerate}
  \item Did you state the full set of assumptions of all theoretical results?
    \answerYes{}
        \item Did you include complete proofs of all theoretical results?
    \answerYes{Some of the proofs are in the Appendix.}
\end{enumerate}

\item If you ran experiments...
\begin{enumerate}
  \item Did you include the code, data, and instructions needed to reproduce the main experimental results (either in the supplemental material or as a URL)?
    \answerYes{The code and instructions for running it are available in the supplementary materials. The code does not use any datasets.}
  \item Did you specify all the training details (e.g., data splits, hyperparameters, how they were chosen)?
    \answerNA{We do not perform model training in this work.}
        \item Did you report error bars (e.g., with respect to the random seed after running experiments multiple times)?
    \answerNA{}
        \item Did you include the total amount of compute and the type of resources used (e.g., type of GPUs, internal cluster, or cloud provider)?
    \answerNA{No compute beyond a personal laptop (with integrated graphics) was used.}
\end{enumerate}

\item If you are using existing assets (e.g., code, data, models) or curating/releasing new assets...
\begin{enumerate}
  \item If your work uses existing assets, did you cite the creators?
    \answerNA{The codebase was written from scratch.}
  \item Did you mention the license of the assets?
    \answerNA{}
  \item Did you include any new assets either in the supplemental material or as a URL?
    \answerYes{The codebase is available in the supplemental material.}
  \item Did you discuss whether and how consent was obtained from people whose data you're using/curating?
    \answerNA{}
  \item Did you discuss whether the data you are using/curating contains personally identifiable information or offensive content?
    \answerNA{}
\end{enumerate}

\item If you used crowdsourcing or conducted research with human subjects...
\begin{enumerate}
  \item Did you include the full text of instructions given to participants and screenshots, if applicable?
    \answerNA{}
  \item Did you describe any potential participant risks, with links to Institutional Review Board (IRB) approvals, if applicable?
    \answerNA{}
  \item Did you include the estimated hourly wage paid to participants and the total amount spent on participant compensation?
    \answerNA{}
\end{enumerate}

\end{enumerate}

\newpage

\appendix

\setcounter{proposition}{0}
\setcounter{theorem}{0}
\setcounter{lemma}{0}
\setcounter{corollary}{0}
\setcounter{definition}{0}

\section{Overview}
Section~\ref{sec:proofs} contains proofs of the main theoretical results.
Section~\ref{sec:examples} expands on examples given in the main text.
Section~\ref{sec:hackability_diagram} presents an unhackability diagram for a generic set of three policies $a, b, c$; Section~\ref{sec:simplification_diagram} shows a simplification diagram of the same policies.

\section{Proofs} \label{sec:proofs}

Before proving our results, we restate assumptions and definitions.
First, recall the preliminaries from Section 4.1, and in particular, that we use $\mathcal{F}: \Pi \rightarrow \mathbb{R}^{|S||A|}$ to denote the embedding of policies into Euclidean space via their discounted state-action visit counts, i.e.;
$$
\mathcal{F}(\pi)[s,a] = \sum_{t=0}^\infty \gamma^t \mathbb{P}(S_t = s, A_t = a).
$$
Given a reward function $\mathcal{R}$, let $\vec{\mathcal{R}} \in \mathbb{R}^{|S||A|}$ be the vector where $\vec{\mathcal{R}}[s,a] = \mathbb{E}_{S' \sim T(s,a)}[\mathcal{R}(s,a,S')]$. Note that $J(\pi) = \mathcal{F}(\pi)\cdot \vec{\mathcal{R}}$.

We say $\mathcal{R}_1$ and $\mathcal{R}_2$ are \textbf{equivalent} on a set of policies $\Pi$ if $J_1$ and $J_2$ induce the same ordering of $\Pi$, and that $\mathcal{R}$ is \textbf{trivial} on $\Pi$ if $J(\pi) = J(\pi')$ for all $\pi,\pi' \in \Pi$. 
We also have the following definitions from Sections 4 and 5:

\begin{definition}\label{def:unhackable}
A pair of reward functions $\mathcal{R}_1$, $\mathcal{R}_2$ are \textbf{hackable} relative to policy set $\Pi$ and an environment $(S,A,T,I,\underline{\hspace*{0.3cm}},\gamma)$ if 
there exist $\pi,\pi' \in \Pi$ such that  
$$
J_1(\pi) < J_1(\pi') \And J_2(\pi) > J_2(\pi'),
$$
else they are \textbf{unhackable}.
\end{definition}

\begin{definition}
$\mathcal{R}_2$ is a \textbf{simplification} of $\mathcal{R}_1$ relative to policy set $\Pi$ if for all $\pi,\pi' \in \Pi$, 
$$
J_1(\pi) < J_1(\pi') \implies J_2(\pi) \leq J_2(\pi') 
\And 
J_1(\pi) = J_1(\pi') \implies J_2(\pi) = J_2(\pi')
$$
and there exist $\pi,\pi' \in \Pi$ such that $J_2(\pi) = J_2(\pi')$ but $J_1(\pi) \neq J_1(\pi')$. Moreover, if $\mathcal{R}_2$ is trivial
then we say that this is a \textbf{trivial simplification}.
\end{definition}

Note that these definitions only depend on the policy orderings associated with $\mathcal{R}_2$ and $\mathcal{R}_1$, and so we can (and do) also speak of (ordered) pairs of policy orderings being simplifications or hackable. We also make use of the following definitions:

\begin{definition}
A (stationary) policy $\pi$ is $\varepsilon$-suboptimal if $J(\pi) \geq J(\pi^\star) - \varepsilon$, where $\varepsilon > 0$
\end{definition}

\begin{definition}
A (stationary) policy $\pi$ is $\delta$-deterministic if $\forall s \in S \; \exists a \in A: \mathbb{P}(\pi(s) = a) \geq \delta$, where $\delta < 1$.
\end{definition}

\subsection{Non-trivial Unhackability Requires Restricting the Policy Set}\label{app:results_all}

Formally, a set of (stationary) policies $\dot{\Pi}$ is \textbf{open} if $\mathcal{V}(\dot{\Pi})$ is open in the smallest affine space that contains $\mathcal{V}(\Pi)$, where $\Pi$ is the set of all stationary policies. Note that this space is $|S|(|A|-1)$-dimensional, since all action probabilities sum to 1.

We require two more propositions for the proof of this lemma.

\begin{proposition}\label{prop:injectivity}
If $\dot{\Pi}$ is open then $\mathcal{F}$ is injective on $\dot{\Pi}$.
\end{proposition}
\begin{proof}
First note that, since $\pi(a \mid s) \geq 0$, we have that if $\dot{\Pi}$ is open then $\pi(a \mid s) > 0$ for all $s,a$ for all $\pi \in \dot{\Pi}$. In other words, all policies in $\dot{\Pi}$ take each action with positive probability in each state. 

Now suppose $\mathcal{F}(\pi) = \mathcal{F}(\pi')$ for some $\pi,\pi' \in \tilde{\Pi}$. Next, define $w_\pi$ as
$$
w_\pi(s) = \sum_{t=0}^\infty \gamma^t \mathbb{P}_{\tau \sim \pi}(S_t = s).
$$
Note that if $\mathcal{F}(\pi) = \mathcal{F}(\pi')$ then $w_\pi = w_{\pi'}$, and moreover that
$$
\mathcal{F}(\pi)[s,a] = w_\pi(s)\pi(a \mid s).
$$
Next, since $\pi$ takes each action with positive probability in each state, we have that $\pi$ visits every state with positive probability. This implies that $w_{\pi}(s) \neq 0$ for all $s$, which means that we can express $\pi$ as
$$
\pi(a \mid s) = \frac{\mathcal{F}(\pi)[s,a]}{w_\pi(s)}.
$$
This means that if $\mathcal{F}(\pi) = \mathcal{F}(\pi')$ for some $\pi,\pi' \in \tilde{\Pi}$ then $\pi = \pi'$.
\end{proof}

Note that $\mathcal{F}$ is \emph{not} injective on $\Pi$; if there is some state $s$ that $\pi$ reaches with probability $0$, then we can alter the behaviour of $\pi$ at $s$ without changing $\mathcal{F}(\pi)$. But every policy in an open policy set $\dot{\Pi}$ visits every state with positive probability, which then makes $\mathcal{F}$ injective. In fact, Proposition~\ref{prop:injectivity} straightforwardly generalises to the set of all policies that visit all states with positive probability (although this will not be important for our purposes).%

\begin{proposition}\label{prop:image_dimension}
$\mathrm{Im}(\mathcal{F})$ is located in an affine subspace with $|S|(|A|-1)$ dimensions.
\end{proposition}
\begin{proof}
To show that $\mathrm{Im}(\mathcal{F})$ is located in an affine subspace with $|S|(|A|-1)$ dimensions, first note that
$$
\sum_{s,a} \mathcal{F}(\pi)[s,a] = \sum_{t=0}^\infty \gamma^t = \frac{1}{1-\gamma}
$$
for all $\pi$. That is, $\mathrm{Im}(\mathcal{F})$ is located in an affine space of points with a fixed $\ell_1$-norm, and this space does not contain the origin.

Next, note that $J(\pi) = \mathcal{F}(\pi) \cdot \vec{\mathcal{R}}$. This means that if knowing the value of $J$ for all $\pi$ determines $\vec{\mathcal{R}}$ modulo at least $n$ free variables, then $\mathrm{Im}(\mathcal{F})$ contains at most $|S||A|-n$ linearly independent vectors.
Next recall \emph{potential shaping} \citep{ng1999policy}. In brief, given a reward function $\mathcal{R}$ and a \emph{potential function} $\Phi : S \rightarrow \mathbb{R}$, we can define a \emph{shaped reward function} $\mathcal{R}'$ by
$$
\mathcal{R}'(s,a,s') = \mathcal{R}(s,a,s') + \gamma\Phi(s') - \Phi(s),
$$
or, alternatively, if we wish $\mathcal{R}'$ to be defined over the domain $S \times A$,
$$
\mathcal{R}'(s,a) = \mathcal{R}(s,a) + \gamma \mathbb{E}_{S' \sim T(s,a)}[\Phi(S')] - \Phi(s).
$$
In either case, it is possible to show that if $\mathcal{R}'$ is produced by shaping $\mathcal{R}$ with $\Phi$, and $\mathbb{E}_{S_0 \sim I}\left[\Phi(S_0)\right] = 0$, then $J(\pi) = J'(\pi)$ for all $\pi$. This means that knowing the value of $J(\pi)$ for all $\pi$ determines $\vec{\mathcal{R}}$ modulo at least $|S|-1$ free variables, which means that $\mathrm{Im}(\mathcal{F})$ contains at most $|S||A| - (|S| - 1) = |S|(|A|-1) + 1$ linearly independent vectors. Since the smallest affine space that contains $\mathrm{Im}(\mathcal{F})$ does \emph{not} contain the origin, this in turn means that $\mathrm{Im}(\mathcal{F})$ is located in an affine subspace with $= |S|(|A|-1) + 1 - 1 = |S|(|A|-1)$ dimensions.
\end{proof}

\begin{lemma}%
In any $MDP \setminus \mathcal{R}$, if $\dot{\Pi}$ is an open set of policies, then
$\mathcal{F}(\dot{\Pi})$ is open in $\mathbb{R}^{|S|(|A|-1)}$, and $\mathcal{F}$ is a homeomorphism between $\mathcal{V}(\dot{\Pi})$ and $\mathcal{F}(\dot{\Pi})$.
\end{lemma}
\begin{proof}
By the Invariance of Domain Theorem, if 
\begin{enumerate}
    \item $U$ is an open subset of $\mathbb{R}^n$, and
    \item $f : U \rightarrow \mathbb{R}^n$ is an injective continuous map,
\end{enumerate}
then $f(U)$ is open in $\mathbb{R}^n$ and $f$ is a homeomorphism between $U$ and $f(U)$. We will show that $\mathcal{F}$ and $\dot{\Pi}$ satisfy the requirements of this theorem.

We begin by noting that $\dot{\Pi}$ can be represented as a set of points in $\mathbb{R}^{|S|(|A|-1)}$. 
First, project $\dot{\Pi}$ into $\mathbb{R}^{|S||A|}$ via $\mathcal{V}$. Next, since $\sum_{a \in A} \pi(a \mid s) = 1$ for all $s$, $\mathrm{Im}(\mathcal{V})$ is in fact located in an affine subspace with $|S|(|A|-1)$ dimensions, which directly gives a representation in $\mathbb{R}^{|S|(|A|-1)}$. Concretely, this represents each policy $\pi$ as a vector $\mathcal{V}(\pi)$ with one entry containing the value $\pi(a \mid s)$ for each state-action pair $s,a$, but with one action left out for each state, since this value can be determined from the remaining values.
We will assume that $\dot{\Pi}$ is embedded in $\mathbb{R}^{|S|(|A|-1)}$ in this way.

By assumption, $\mathcal{V}(\dot{\Pi})$ is an open set in $\mathbb{R}^{|S|(|A|-1)}$. Moreover, by Proposition~\ref{prop:image_dimension}, we have that $\mathcal{F}$ is (isomorphic to) a mapping $\dot{\Pi} \rightarrow \mathbb{R}^{|S|(|A|-1)}$. By Proposition~\ref{prop:injectivity}, we have that $\mathcal{F}$ is injective on $\dot{\Pi}$. Finally, $\mathcal{F}$ is continuous; this can be seen from its definition. We can therefore apply the Invariance of Domain Theorem, and obtain that $\mathcal{F}(\dot{\Pi})$ is open in $\mathbb{R}^{|S|(|A|-1)}$, and that $\mathcal{F}$ is a homeomorphism between $\mathcal{V}(\dot{\Pi})$ and $\mathcal{F}(\dot{\Pi})$.
\end{proof}

\begin{figure*}[ht]
\centering
\includegraphics[width=.8\textwidth]{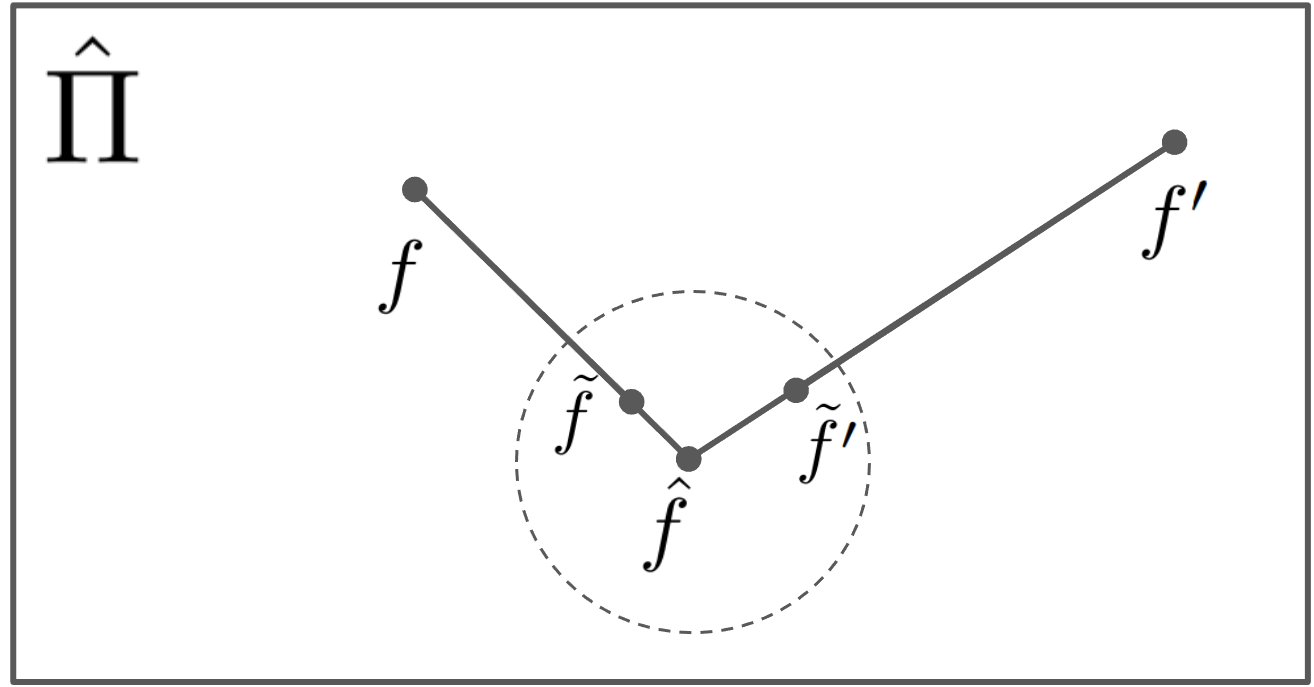}
\caption{Illustration of the various realizable feature counts used in the proof of Theorem 1.}
\label{fig:open_ball}
\end{figure*}

\begin{theorem} %
In any $MDP \setminus \mathcal{R}$, if $\hat{\Pi}$ contains an open set, then any pair of reward functions that are unhackable and non-trivial on $\hat{\Pi}$ are equivalent on $\hat{\Pi}$.
\end{theorem}
\begin{proof}

Let $\mathcal{R}_1$ and $\mathcal{R}_2$ be any two unhackable and non-trivial reward functions.
We will show that, for any $\pi, \pi' \in \hat{\Pi}$, we have $J_1(\pi) = J_1(\pi') \implies J_2(\pi) = J_2(\pi')$, and thus, by symmetry, $J_1(\pi) = J_1(\pi') \iff J_2(\pi) = J_2(\pi')$.
Since $\mathcal{R}_1$ and $\mathcal{R}_2$ are unhackable, this further means that they have exactly the same policy order, i.e.\ that they are equivalent. %

Choose two arbitrary $\pi, \pi' \in \hat{\Pi}$ with $J_1(\pi) = J_1(\pi')$ and let $f \doteq \mathcal{F}(\pi), f' \doteq \mathcal{F}(\pi')$.
The proof has 3 steps:%
\begin{enumerate}
    \item We find analogues for $f$ and $f'$, $\tilde{f}$ and $\tilde{f'}$, within the same open ball in $\mathcal{F}(\hat{\Pi})$. %
    \item We show that the tangent hyperplanes of $\vec{R}_1$ and $\vec{R}_2$ at $\tilde{f}$ must be equal to prevent neighbors of $\tilde{f}$ from making $\mathcal{R}_1$ and $\mathcal{R}_2$ hackable.
    \item We use linearity to show that this implies that $J_2(\pi) = J_2(\pi')$.
\end{enumerate}

\textbf{Step 1:}
By assumption, $\hat{\Pi}$ contains an open set $\dot{\Pi}$. Let $\hat{\pi}$ be some policy in $\dot{\Pi}$, and let $\hat{f} \doteq \mathcal{F}(\hat{\pi})$. Since $\dot{\Pi}$ is open, Lemma~\ref{lemma:homeomorphism} implies that $\mathcal{F}(\dot{\Pi})$ is open in $\mathbb{R}^{|S|(|A|-1)}$.
This means that, if $v, v'$ are the vectors such that $\hat{f} + v = f$ and $\hat{f} + v' = f'$, then there is a positive but sufficiently small $\delta$ such that $\tilde{f} \doteq \hat{f} + \delta v$ and $\tilde{f'} \doteq \hat{f} + \delta v'$ both are located in $\mathcal{F}(\dot{\Pi})$, see Figure~\ref{fig:open_ball}.
This further implies that there are policies $\tilde{\pi}, \tilde{\pi'} \in \dot{\Pi}$ such that $\mathcal{F}(\tilde{\pi}) = \tilde{f}$ and $\mathcal{F}(\tilde{\pi'}) = \tilde{f'}$.

\textbf{Step 2:}
Recall that $J(\pi) = \mathcal{F}(\pi) \cdot \vec{\mathcal{R}}$. %
Since $\mathcal{R}_1$ is non-trivial on $\hat{\Pi}$, it induces a \linebreak $(|S|(|A|-1) - 1)$-dimensional hyperplane tangent to $\vec{\mathcal{R}}_1$ corresponding to all points $x \in \mathbb{R}^{|S|(|A|-1)}$ such that $x \cdot \vec{\mathcal{R}}_1 = \tilde{f} \cdot \vec{\mathcal{R}}_1$, and similarly for $\mathcal{R}_2$.
Call these hyperplanes $H_1$ and $H_2$, respectively.
Note that $\tilde{f}$ is contained in both $H_1$ and $H_2$.

Next suppose $H_1 \neq H_2$.
Then, we would be able to find a point $f_{12} \in \mathcal{F}(\dot{\Pi})$, such that $f_{12} \cdot \vec{\mathcal{R}}_1 > \tilde{f} \cdot \vec{\mathcal{R}}_1$ but $f_{12} \cdot \vec{\mathcal{R}}_2 < \tilde{f} \cdot \vec{\mathcal{R}}_2$. This, in turn, means that there is a policy $\pi_{12} \in \dot{\Pi}$ such that $\mathcal{F}(\pi_{12}) = f_{12}$, and such that $J_1(\pi_{12}) > J_1(\tilde{\pi})$ but $J_2(\pi_{12}) < J_2(\tilde{\pi})$. Since $\mathcal{R}_1$ and $\mathcal{R}_2$ are unhackable, this is a contradiction. Thus $H_1 = H_2$.

\textbf{Step 3:}
Since $J_1(\pi) = J_1(\pi')$, we have that $f \cdot \vec{\mathcal{R}}_1 = f' \cdot \vec{\mathcal{R}}_1$. 
By linearity, this implies that $\tilde{f}  %
\cdot \vec{\mathcal{R}}_1 = \tilde{f}' \cdot \vec{\mathcal{R}}_1$; we can see this by expanding $\tilde{f} = \hat{f} + \delta v$
and $\tilde{f'} = \hat{f} + \delta v'$. 
This means that $\tilde{f}' \in H_1$. 
Now, since $H_1 = H_2$, this means that $\tilde{f}' \in H_2$, which in turn implies that $\tilde{f} \cdot \vec{\mathcal{R}}_2 = \tilde{f}' \cdot \vec{\mathcal{R}}_2$. By linearity, this then further implies that $f \cdot \vec{\mathcal{R}}_2 = f' \cdot \vec{\mathcal{R}}_2$, and hence that $J_2(\pi) = J_2(\pi')$. 
Since $\pi,\pi'$ were chosen arbitrarily, this means that $J_1(\pi) = J_1(\pi') \implies J_2(\pi) = J_2(\pi')$.
\end{proof}

\begin{corollary}
In any $MDP \setminus \mathcal{R}$, any pair of reward functions that are unhackable and non-trivial on the set of all (stationary) policies $\Pi$ are equivalent on $\Pi$.
\end{corollary}
\begin{proof}
This corollary follows from Theorem~\ref{thm:open-set}, if we note that the set of all policies does contain an open set. This includes, for example, the set of all policies in an $\epsilon$-ball around the policy that takes all actions with equal probability in each state.
\end{proof}

\begin{corollary}
In any $MDP \setminus \mathcal{R}$, any pair of reward functions that are unhackable and non-trivial on the set of all $\varepsilon$-suboptimal policies ($\varepsilon>0$) $\Pi^\varepsilon$ are equivalent on $\Pi^\varepsilon$, and any pair of reward functions that are unhackable and non-trivial on the set of all $\delta$-deterministic policies ($\delta<1$) $\Pi^\delta$ are equivalent on $\Pi^\delta$.
\end{corollary}
\begin{proof}
To prove this, we will establish that both $\Pi^\varepsilon$ and $\Pi^\delta$ contain open policy sets, and then apply Theorem~\ref{thm:open-set}.

Let us begin with $\Pi^\delta$. First, let $\pi$ be some deterministic policy, and let $\pi_\epsilon$ be the policy that in each state with probability $1-\epsilon$ takes the same action as $\pi$, and otherwise samples an action uniformly. Then if $\delta < \epsilon < 1$, $\pi_\epsilon$ is the center of an open ball in $\Pi^\delta$. Thus $\Pi^\delta$ contains an open set, and we can apply Theorem~\ref{thm:open-set}.

For $\Pi^\varepsilon$, let $\pi^\star$ be an optimal policy, and apply an analogous argument.
\end{proof}

\subsection{Finite Policy Sets}

\begin{theorem}%
For any $MDP \setminus \mathcal{R}$, any finite set of policies $\hat{\Pi}$ containing at least two $\pi,\pi'$ such that $\mathcal{F}(\pi) \neq \mathcal{F}(\pi')$, and any reward function $\mathcal{R}_1$, there is a non-trivial reward function $\mathcal{R}_2$ such that $\mathcal{R}_1$ and $\mathcal{R}_2$ are unhackable but not equivalent.
\end{theorem}

\begin{proof}
If $\mathcal{R}_1$ is trivial, then simply choose any non-trivial $\mathcal{R}_2$.
Otherwise, the proof proceeds by finding a path from $\vec{\mathcal{R}}_1$ to $-\vec{\mathcal{R}}_1$, and showing that there must be an $\vec{\mathcal{R}}_2$ on this path such that $\mathcal{R}_2$ is non-trivial and unhackable with respect to $\mathcal{R}_1$, but not equivalent to $\mathcal{R}_1$.

The key technical difficulty is to show that there exists a continuous path from $\mathcal{R}_1$ to $-\mathcal{R}_1$ in $\mathbb{R}^{|S||A|}$ that does not include any trivial reward functions. 
Once we've established that, we can simply look for the first place where an inequality is reversed -- because of continuity, it first becomes an equality. 
We call the reward function at that point $\mathcal{R}_2$, and note that $\mathcal{R}_2$ is unhackable wrt $\mathcal{R}_1$ and not equivalent to $\mathcal{R}_1$. 
We now walk through the technical details of these steps.

First, note that $J(\pi) = \mathcal{F}(\pi)\cdot\vec{\mathcal{R}}$ is continuous in $\vec{\mathcal{R}}$. This means that if $J_1(\pi) > J_2(\pi')$ then there is a unique first vector $\vec{\mathcal{R}}_2$ on any path from $\vec{\mathcal{R}}_1$ to $-\vec{\mathcal{R}}_1$ such that $\mathcal{F}(\pi)\cdot\vec{\mathcal{R}}_2 \not> \mathcal{F}(\pi)\cdot\vec{\mathcal{R}}_2$, and for this vector we have that $\mathcal{F}(\pi)\cdot\vec{\mathcal{R}}_2 = \mathcal{F}(\pi)\cdot\vec{\mathcal{R}}_2$. Since $\hat{\Pi}$ is finite, and since $\mathcal{R}_1$ is not trivial, this means that on any path from $\vec{\mathcal{R}}_1$ to $-\vec{\mathcal{R}}_1$ there is a unique first vector $\vec{\mathcal{R}}_2$ such that $\mathcal{R}_2$ is not equivalent to $\mathcal{R}_1$, and then $\mathcal{R}_2$ must also be a unhackable with respect to $\mathcal{R}_1$.

It remains to show that there is a path from $\vec{\mathcal{R}}_1$ to $-\vec{\mathcal{R}}_1$ such that no vector along this path corresponds to a trivial reward function. 
Once we have such a path, the argument above implies that $\mathcal{R}_2$ must be a non-trivial reward function that is unhackable with respect to $\mathcal{R}_1$.
We do this using a dimensionality argument.
If $\mathcal{R}$ is trivial on $\hat{\Pi}$, then there is some $c \in \mathbb{R}$ such that $\mathcal{F}(\pi)\cdot\vec{\mathcal{R}} = c$ for all $\pi \in \hat{\Pi}$. This means that if $\mathcal{F}(\hat{\Pi})$ has at least $d$ linearly independent vectors, then the set of all such vectors $\vec{\mathcal{R}}$ forms a linear subspace with at most $|S||A| - d$ dimensions.
Now, since $\hat{\Pi}$ contains at least two $\pi,\pi'$ such that $\mathcal{F}(\pi) \neq \mathcal{F}(\pi')$, we have that $\mathcal{F}(\hat{\Pi})$ has at least $2$ linearly independent vectors, and hence that the set of all reward functions that are trivial on $\hat{\Pi}$ forms a linear subspace with at most $|S||A| - 2$ dimensions.
This means that there must exist a path from $\vec{\mathcal{R}}_1$ to $-\vec{\mathcal{R}}_1$ that avoids this subspace, 
since only a hyperplane (with dimension $|S||A| - 1$) can split $\mathbb{R}^{|S||A|}$ into two disconnected components.
\end{proof}

\begin{theorem}%
Let $\hat{\Pi}$ be a finite set of policies, and $\mathcal{R}$ a reward function. The following procedure determines if there exists a non-trivial simplification of $\mathcal{R}$ in a given $MDP \setminus \mathcal{R}$:
\begin{enumerate}
    \item Let $E_1 \dots E_m$ be the partition of $\hat{\Pi}$ where $\pi,\pi'$ belong to the same set iff $J(\pi) = J(\pi')$.
    \item For each such set $E_i$, select a policy $\pi_i \in E_i$ and let $Z_i$ be the set of vectors that is obtained by subtracting $\mathcal{F}(\pi_i)$ from each element of $\mathcal{F}(E_i)$.
\end{enumerate}
Then there is a non-trivial simplification of $\mathcal{R}$ iff $\mathrm{dim}(Z_1 \cup \dots \cup Z_m) \leq \mathrm{dim}(\mathcal{F}(\hat{\Pi})) - 2$, where $\mathrm{dim}(S)$ is the number of linearly independent vectors in $S$.
\end{theorem}

\begin{proof}
This proof uses a similar proof strategy as Theorem~\ref{thm:finite_unhackability}. However, in addition 
to avoiding trivial reward functions on the path from $\vec{\mathcal{R}}_1$ to $-\vec{\mathcal{R}}_1$, we must also ensure that we stay within the ``equality-preserving space'', to be defined below. 

First recall that $\mathcal{F}(\hat{\Pi})$ is a set of vectors in $\mathbb{R}^{|S||A|}$. If $\mathrm{dim}(\mathcal{F}(\hat{\Pi})) = D$ then these vectors are located in a $D$-dimensional linear subspace. Therefore, we will consider $\mathcal{F}(\hat{\Pi})$ to be a set of vectors in $\mathbb{R}^D$. Next, recall that any reward function $\mathcal{R}$ induces a linear function $L$ on $\mathbb{R}^D$, such that $J = L \circ \mathcal{F}$, and note that there is a $D$-dimensional vector $\vec{\mathcal{R}}$ that determines the \emph{ordering} that $\mathcal{R}$ induces over 
all points in 
$\mathbb{R}^D$. 
To determine the \emph{values} of $J$ on all points in $\mathbb{R}^D$ we would need a $(D+1)$-dimensional vector, but to determine the \emph{ordering}, we can ignore the height of the function. In other words, $L(x) = x \cdot \vec{\mathcal{R}} + L(\vec{0})$, for any $x \in \mathbb{R}^D$. Note that this is a different vector representation of reward functions than that which was used in Theorem~\ref{thm:finite_unhackability} and before.

Suppose $\mathcal{R}_2$ is a reward function such that if $J_1(\pi) = J_1(\pi')$ then $J_2(\pi) = J_2(\pi')$, for all $\pi,\pi' \in \hat{\Pi}$. This is equivalent to saying that $L_2(\mathcal{F}(\pi)) =  L_2(\mathcal{F}(\pi'))$ if $\pi,\pi' \in E_i$ for some $E_i$.
By the properties of linear functions, this implies that if $\mathcal{F}(E_i)$ contains $d_i$ linearly independent vectors then it specifies a $(d_i-1)$-dimensional affine space $S_i$ such that $L_2(x) = L_2(x')$ for all points $x,x' \in S_i$. Note that this is the smallest affine space which contains all points in $E_i$.
Moreover, $L_2$ is also constant for any affine space $\bar{S_i}$ \emph{parallel} to $S_i$. 
Formally, we say that $\bar{S_i}$ is parallel to $S_i$ if there is a vector $z$ such that for any $y \in \bar{S_i}$ there is an $x \in S_i$ such that $y = x + z$. From the properties of linear functions, if $L_2(x) = L_2(x')$ then $L_2(x + z) = L_2(x' + z)$.

Next, from the transitivity of equality, if we have two affine spaces $\bar{S_i}$ and $\bar{S_j}$, such that $L_2$ is constant over each of $\bar{S_i}$ and $\bar{S_j}$, and such that $\bar{S_i}$ and $\bar{S_j}$ \emph{intersect}, then $L_2$ is constant over all points in $\bar{S_i} \cup \bar{S_j}$. From the properties of linear functions, this then implies that $L_2$ is constant over all points in the smallest affine space $\bar{S_i}\otimes\bar{S_j}$ containing $\bar{S_i}$ and $\bar{S_j}$, given by combining the linearly independent vectors in $\bar{S_i}$ and $\bar{S_j}$.
Note that $\bar{S_i}\otimes\bar{S_j}$ has between $\mathrm{max}(d_i,d_j)$ and $(d_i + d_j - 1)$ dimensions.
In particular, since the affine spaces of $Z_1 \dots Z_m$ intersect (at the origin), and since $L_2$ is constant over these spaces, we have that $L_2$ must be constant for all points in the affine space $\mathcal{Z}$ 
which is the smallest affine space containing
$Z_1 \cup \dots \cup Z_m$. That is, if $\mathcal{R}_2$ is a reward function such that $J_1(\pi) = J_1(\pi') \implies J_2(\pi) = J_2(\pi')$ for all $\pi,\pi' \in \hat{\Pi}$, then $L_2$ is constant over $\mathcal{Z}$. Moreover, if $L_2$ is constant over $\mathcal{Z}$ then $L_2$ is also constant over each of $E_1 \dots E_m$, since each of $E_1 \dots E_m$ is parallel to $\mathcal{Z}$. This means that $\mathcal{R}_2$ satisfies that $J_1(\pi) = J_1(\pi') \implies J_2(\pi) = J_2(\pi')$ for all $\pi,\pi' \in \hat{\Pi}$ if and only if $L_2$ is constant over $\mathcal{Z}$.

If $\mathrm{dim}(\mathcal{Z}) = D'$ then there is a linear subspace with $D - D'$ dimensions, which contains the ($D$-dimensional) vector $\vec{\mathcal{R}}_2$ of any reward function $\mathcal{R}_2$ where $J_1(\pi) = J_1(\pi') \implies J_2(\pi) = J_2(\pi')$ for $\pi,\pi' \in \hat{\Pi}$. This is because $\mathcal{R}_2$ is constant over $\mathcal{Z}$ if and only if $\vec{R}_2 \cdot v = 0$ for all $v \in \mathcal{Z}$. Then if $\mathcal{Z}$ contains $D'$ linearly independent vectors $v_i \dots v_{D'}$, then the solutions to the corresponding system of linear equations form a $(D - D')$ dimensional subspace of $\mathbb{R}^D$.
Call this space the \emph{equality-preserving space}. Next, note that $\mathcal{R}_2$ is trivial on $\hat{\Pi}$ if and only if $\vec{\mathcal{R}}_2$ is the zero vector $\vec{0}$.

Now we show that if the conditions are not satisfied, then there is no non-trivial simplification.
Suppose $D' \geq D-1$, and that $\mathcal{R}_2$ is a simplification of $\mathcal{R}_1$. Note that if $\mathcal{R}_2$ simplifies $\mathcal{R}_1$ then $\vec{\mathcal{R}}_2$ is in the equality-preserving space.
Now, if $D' = D$ then $L_2$ (and $L_1$) must be constant for all points in $\mathbb{R}^D$, which implies that $\mathcal{R}_2$ (and $\mathcal{R}_1$) are trivial on $\hat{\Pi}$. Next, if $D' = D - 1$ then the equality-preserving space is one-dimensional.
Note that we can always preserve all equalities of $\mathcal{R}_1$ by \emph{scaling} $\mathcal{R}_1$ by a constant factor. That is, if $\mathcal{R}_2 = c \cdot \mathcal{R}_1$ for some (possibly negative) $c \in \mathbb{R}$ then $J_1(\pi) = J_1(\pi') \implies J_2(\pi) = J_2(\pi')$ for all $\pi,\pi' \in \hat{\Pi}$. This means that the parameter which corresponds to the dimension of the equality-preserving space in this case must be the scaling of $\vec{\mathcal{R}}_2$. However, the only simplification of $\mathcal{R}_1$ that is obtainable by uniform scaling is the trivial simplification. This means that if $D' \geq D-1$ then $\mathcal{R}_1$ has no non-trivial simplifications on $\hat{\Pi}$.

For the other direction, suppose $D' \leq D-2$. Note that this implies that $\mathcal{R}_1$ is not trivial.
Let $\mathcal{R}_3 = -\mathcal{R}_1$. Now both $\vec{\mathcal{R}}_1$ and $\vec{\mathcal{R}}_3$ are located in the equality-preserving space. Next, since the equality-preserving space has at least two dimensions, this means that there is a continuous path from $\vec{\mathcal{R}}_1$ to $\vec{\mathcal{R}}_3$ through the equality-preserving space that does not pass the origin.
Now, note that $J_i(\pi) = \mathcal{F}(\pi)\cdot\vec{\mathcal{R}}_i$ is continuous in $\vec{\mathcal{R}}_i$.
This means that there, on the path from $\vec{\mathcal{R}}_1$ to $\vec{\mathcal{R}}_3$ is a first vector $\vec{\mathcal{R}}_2$ such that $\mathcal{F}(\pi)\cdot\vec{\mathcal{R}}_2 = \mathcal{F}(\pi')\cdot\vec{\mathcal{R}}_2$ but $\mathcal{F}(\pi)\cdot\vec{\mathcal{R}}_1 \neq \mathcal{F}(\pi')\cdot\vec{\mathcal{R}}_1$ for some $\pi,\pi' \in \hat{\Pi}$. 
Let $\mathcal{R}_2$ be a reward function corresponding to $\vec{\mathcal{R}_2}$. Since $\vec{\mathcal{R}_2}$ is not $\vec{0}$, we have that $\mathcal{R}_2$ is not trivial on $\hat{\Pi}$. Moreover, since $\vec{\mathcal{R}_2}$ is in the equality-preserving space, and since $\mathcal{F}(\pi)\cdot\vec{\mathcal{R}}_2 = \mathcal{F}(\pi')\cdot\vec{\mathcal{R}}_2$ but $\mathcal{F}(\pi)\cdot\vec{\mathcal{R}}_1 \neq \mathcal{F}(\pi')\cdot\vec{\mathcal{R}}_1$ for some $\pi,\pi' \in \hat{\Pi}$, we have that $\mathcal{R}_2$ is a non-trivial simplification of $\mathcal{R}_1$. Therefore, if $D' \leq D-2$ then there exists a non-trivial simplification of $\mathcal{R}_1$.

We have thus proven both directions, which completes the proof.
\end{proof}

\begin{corollary}
For any finite set of policies $\hat{\Pi}$, any environment, and any reward function $\mathcal{R}$, if $|\hat{\Pi}| \geq 2$ and $J(\pi) \neq J(\pi')$ for all $\pi,\pi' \in \hat{\Pi}$, then there is a non-trivial simplification of $\mathcal{R}$.
\end{corollary}
\begin{proof}
Note that if $E_i$ is a singleton set then $Z_i = \{\vec{0}\}$. Hence, if each $E_i$ is a singleton set then $\mathrm{dim}(Z_1 \cup \dots \cup Z_m)$ = 0. If $\hat{\Pi}$ contains at least two $\pi,\pi'$, and $J(\pi) \neq J(\pi')$, then  $\mathcal{F}(\pi) \neq \mathcal{F}(\pi')$. This means that $\mathrm{dim}(\mathcal{F}(\hat{\Pi})) \geq 2$. Thus the conditions of Theorem~\ref{thm:finite_simplification} are satisfied.
\end{proof}

\section{Any Policy Can Be Made Optimal}\label{sec:any_policy_optimal}

In this section, we show that any policy is optimal under some reward function.

\begin{proposition}
For any rewardless MDP $(S,A,T,I,\underline{\hspace*{0.3cm}},\gamma)$ and any policy $\pi$, there exists a reward function $\mathcal{R}$ such that $\pi$ is optimal in the corresponding MDP $(S,A,T,I,\mathcal{R},\gamma)$.
\end{proposition}
\begin{proof}
Let $\mathcal{R}(s,a,s') = 0$ if $a \in \mathrm{Support}(\pi(s))$, and $-1$ otherwise.
\end{proof}

This shows that any policy is rationalised by some reward function in any environment.
Any policy that gives $0$ probability to any action which $\pi$ takes with $0$ probability is optimal under this construction. This means that if $\pi$ is deterministic, then it will be the only optimal policy in $(S,A,T,I,\mathcal{R},\gamma)$.

\section{Examples}
\label{sec:examples}

In this section, we take a closer look at two previously-seen examples: the two-state $MDP\setminus \mathcal{R}$ and the cleaning robot.

\subsection{Two-state $MDP\setminus \mathcal{R}$ example}
Let us explore in more detail the two-state system introduced in the main text.
We decsribe this infinite-horizon $MDP \setminus \mathcal{R}$ in Table \ref{tab:two_state}.

\begin{table}[h!]
\centering
\renewcommand{\arraystretch}{1.2}
\begin{tabular}{| l  l |}
\hline
 States & $S = \{0, 1\}$  \\ 
 \hline
 Actions & $A = \{0, 1\}$  \\  
 \hline
 Dynamics & $T(s, a) = a$ for $s\in S, a\in A$   \\
 \hline
 Initial state distribution & $\text{Pr}(\text{start in } s) = 0.5$ for $s \in S$ \\
 \hline
 Discount factor & $\gamma = 0.5$ \\
 \hline
\end{tabular}
\vspace{2mm}
\caption{The two-state $MDP\setminus \mathcal{R}$ in consideration.}
\label{tab:two_state}
\end{table}

We denote $\pi_{ij}$ ($i,j \in \{0, 1\}$) the policy
which takes action $i$ when in state 0 and action $j$
when in state 1. This gives us four possible deterministic policies:
\begin{equation*}
    \{ \pi_{00}, \pi_{01}, \pi_{10}, \pi_{11} \}.
\end{equation*}
 
There are $4! = 24$ ways of
ordering these policies with strict inequalities.
Arbitrarily setting 
$\pi_{00} < \pi_{11}$ breaks a symmetry and reduces
the number of policy orderings to 12. 
When a policy ordering can be derived from some reward function $\mathcal{R}$, we say that $\mathcal{R}$ \textbf{represents} it, and that the policy ordering is \textbf{representable}.
Of these 12 policy orderings with strict inequalities, six are representable:
\begin{align*}
   & \pi_{00} < \pi_{01} < \pi_{10} < \pi_{11}, \\ %
   & \pi_{00} < \pi_{01} < \pi_{11} < \pi_{10}, \\ %
   & \pi_{00} < \pi_{10} < \pi_{01} < \pi_{11}, \\ %
   & \pi_{01} < \pi_{00} < \pi_{11} < \pi_{10}, \\ %
   & \pi_{10} < \pi_{00} < \pi_{01} < \pi_{11}, \\ %
   & \pi_{10} < \pi_{00} < \pi_{11} < \pi_{01}. %
\end{align*}

Simplification in this environment is nontrivial
-- given a policy ordering, it is not obvious which strict inequalities can be set to equalities such that there is a reward function which represents the new ordering. Through a computational approach (see Section~\ref{sec:software}) we find the following representable orderings, each of which is a simplification of one of the above strict orderings.
\begin{align*}
    & \pi_{00} = \pi_{01} < \pi_{11} < \pi_{10}, \\ %
    & \pi_{00} = \pi_{10} < \pi_{01} < \pi_{11}, \\ %
    & \pi_{00} < \pi_{01} = \pi_{10} < \pi_{11}, \\ %
    & \pi_{01} < \pi_{00} = \pi_{11} < \pi_{10}, \\ %
    & \pi_{10} < \pi_{00} = \pi_{11} < \pi_{01}, \\ %
    & \pi_{00} < \pi_{01} < \pi_{10} = \pi_{11}, \\ %
    & \pi_{10} < \pi_{00} < \pi_{01} = \pi_{11}, \\ %
    & \pi_{00} = \pi_{01} = \pi_{10} = \pi_{11}.
\end{align*}

Furthermore, for this environment, we find that any reward function which sets the value of three policies equal necessarily forces the value of the fourth policy to be equal as well.

\subsection{Cleaning robot example}

Recall the cleaning robot example in which a robot can choose to clean a combination of three rooms, and receives a nonnegative reward for each room cleaned. This setting can be thought of as a single-step eight-armed bandit with special reward structure.

\subsubsection{Hackability}

We begin our exploration of this environment with a statement regarding exactly when two policies are hackable. In fact, the proposition is slightly more general, extending to an arbitrary (finite) number of rooms.
 
\label{app:cleaning_robot}
\begin{proposition}
    Consider a cleaning robot which can clean 
$N$ different rooms, and identify each room with 
a unique index in \{1, \ldots, N\}. Cleaning
room $i$ gives reward $r(i) \geq 0$. Cleaning multiple
rooms gives reward equal to the
sum of the rewards of the rooms cleaned.
The value
of a policy $\pi_S$ which cleans 
a collection of rooms
$S$
is the sum of the rewards corresponding
to the rooms cleaned:
$J(\pi_S) = \sum_{i \in S} r(i)$.
For room $i$,
the true reward function assigns a value
$r_\text{true}(i)$, while the
proxy reward function assigns it reward $r_\text{proxy}(i)$. 
The proxy reward
is hackable with respect to the true reward if and only if
there are two sets of rooms $S_1, S_2$ such that 
$\sum_{i \in S_1} r_\text{proxy}(i) < \sum_{i \in S_2} r_\text{proxy}(i)$
and 
$\sum_{i \in S_1} r_\text{true}(i) > \sum_{i \in S_2} r_\text{true}(i)$.
\label{prop:cleaning}
\end{proposition}

\begin{proof}
    We show the two directions of the double implication.
    \begin{itemize}
        \item[$\Leftarrow$]
        Suppose there are two sets of rooms $S_1, S_2$ satisfying
        $\sum_{i \in S_1} r_\text{proxy}(i) < \sum_{i \in S_2} r_\text{proxy}(i)$
        and 
        $\sum_{i \in S_1} r_\text{true}(i) > \sum_{i \in S_2} r_\text{true}(i)$.
        The policies
        $\pi_{S_i} = \text{\enquote{clean exactly the rooms in $S_i$}}$
        for $i \in \{1, 2\}$ demonstrate that $r_\text{proxy}, r_\text{true}$ are hackable.
        To see this, remember that $J(\pi_S) = \sum_{i \in S} r(i)$.
        Combining this with the premise immediately gives
        $J_\text{proxy}(\pi_{S_1}) < J_\text{proxy}(\pi_{S_2})$
        and $J_\text{true}(\pi_{S_1}) > J_\text{true}(\pi_{S_2})$.
        
        \item[$\Rightarrow$]
        If $r_\text{proxy}, r_\text{true}$ are hackable, then there must be a pair of policies $\pi_1, \pi_2$ such that $J_\text{proxy}(\pi_1) < J_\text{proxy}(\pi_2)$
        and $J_\text{true}(\pi_1) > J_\text{true}(\pi_2)$. Let $S_1$ be the set
        of rooms cleaned by $\pi_1$ and $S_2$ be the set
        of rooms cleaned by $\pi_2$. 
        Again remembering that
        $J(\pi_S) = \sum_{i \in S} r(i)$
        immediately gives us that
        $\sum_{i \in S_1} r_\text{proxy}(i) < \sum_{i \in S_2} r_\text{proxy}(i)$
        and 
        $\sum_{i \in S_1} r_\text{true}(i) > \sum_{i \in S_2} r_\text{true}(i)$.
    \end{itemize}
\end{proof}

In the main text, we saw two intuitive ways of modifying the reward function in the cleaning robot example: omitting information and overlooking fine details.
Unfortunately, there is no obvious mapping of Proposition~\ref{prop:cleaning} onto simple rules concerning how to safely omit information or overlook fine details: it seems that one must resort to ensuring that no two sets of rooms satisfy the conditions for hackability described in the proposition.

\subsubsection{Simplification}

We now consider simplification in this environment. Since we know the reward for cleaning each room is nonnegative, there will be some structure underneath all the possible orderings over the policies. This structure is shown in Figure \ref{fig:cleaning_policy_ordering}: regardless of the value assigned to each room, a policy at the tail of an arrow can only be at most as good as a policy at the head of the arrow.

\tikzstyle{arrow} = [thick,->,>=stealth]

\begin{figure}[H]
\centering
\begin{tikzpicture}[node distance=2.5cm]
\node (1) [] {[0, 0, 0]};
\node (2) [ right of=1, yshift=0.7cm, xshift=0.65cm] {[0, 0, 1]};
\node (3) [ right of=1, yshift=0cm, xshift=0.65cm] {[0, 1, 0]};
\node (4) [ right of=1, yshift=-0.7cm, xshift=0.65cm] {[1, 0, 0]};
\node (5) [ right of=2, xshift=0.65cm] {[0, 1, 1]};
\node (6) [ right of=3, xshift=0.65cm] {[1, 0, 1]};
\node (7) [ right of=4, xshift=0.65cm] {[1, 1, 0]};
\node (8) [ right of=6, xshift=0.65cm] {[1, 1, 1]};

\draw [arrow] (1) -- (2);
\draw [arrow] (1) -- (3);
\draw [arrow] (1) -- (4);
\draw [arrow] (2) -- (5);
\draw [arrow] (2) -- (6);
\draw [arrow] (3) -- (5);
\draw [arrow] (3) -- (7);
\draw [arrow] (4) -- (6);
\draw [arrow] (4) -- (7);
\draw [arrow] (5) -- (8);
\draw [arrow] (6) -- (8);
\draw [arrow] (7) -- (8);

\end{tikzpicture}
\caption{The structure underlying all possible policy orderings (assuming nonnegative room value). The policy
at the tail of the arrow is at most as good as the policy at
the head of the arrow.}
\label{fig:cleaning_policy_ordering}
\end{figure}
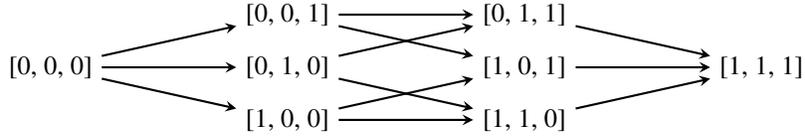

If we decide to simplify an ordering by equating
two policies connected by an arrow, the structure of the reward calculation will force other policies to also be equated.
Specifically,
if the equated policies differ only in position $i$, then all pairs of policies
which differ only in position $i$ will
also be set equal.

For example, imagine we simplify the reward by saying we don't care if the attic is cleaned or not, so long as the other two rooms are cleaned (recall that we named the rooms Attic, Bedroom and Kitchen). This amounts to saying that $J([0, 1, 1]) = J([1, 1, 1])$. 
Because the policy value function is of the form
\begin{equation*}
    J(\pi) = J([x, y, z]) = [x, y, z] \cdot [r_1, r_2, r_3]
\end{equation*}
where $x, y, z \in \{0, 1\}$, this simplification forces $r_1 = 0$. In turn, this implies that
$J([0, 0, 0]) = J([1, 0, 0])$ and $J([0, 1, 0]) = J([1, 1, 0])$. The new structure underlying the ordering over policies is shown in Figure \ref{fig:simplified_cleaning_policy_ordering}.

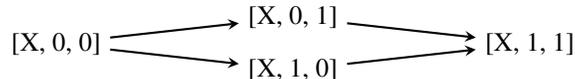
\begin{figure}[H]
\centering
\begin{tikzpicture}[node distance=2.5cm]
\node (1) [] {[X, 0, 0]};
\node (2) [ right of=1, yshift=0.35cm, xshift=0.65cm] {[X, 0, 1]};
\node (3) [ right of=1, yshift=-0.35cm, xshift=0.65cm] {[X, 1, 0]};
\node (4) [ right of=3, yshift=0.35cm, xshift=0.65cm] {[X, 1, 1]};

\draw [arrow] (1) -- (2);
\draw [arrow] (1) -- (3);
\draw [arrow] (2) -- (4);
\draw [arrow] (3) -- (4);

\end{tikzpicture}
\caption{The updated ordering structure after equating \enquote{clean all
the rooms} and \enquote{clean all the rooms except the attic}.
X can take either value in \{0, 1\}.
}
\label{fig:simplified_cleaning_policy_ordering}
\end{figure}

An alternative way to think about simplification in this problem is by imagining policies as corners of a cube, and simplification as flattening of the cube along one dimension -- 
simplification collapses this cube into a square.

\subsection{Software repository}
\label{sec:software}

The software suite described in the paper (and used to calculate the representable policy orderings and simplifications of the two-state $MDP \setminus \mathcal{R}$) can be found at \url{https://github.com/nikihowe/reward-hacking-paper}.

\section{Unhackability Diagram}
\label{sec:hackability_diagram}

Consider a setting with three policies $a, b, c$. We allow all possible orderings of the policies.
In general, these orderings might not all be representable; a concrete case in which they are is when $a, b, c$ represent different deterministic policies in a 3-armed bandit.

We can represent all unhackable pairs of policy orderings with an undirected graph, which we call an \textbf{unhackability diagram}. This includes a node for every representable ordering and edges connecting orderings which are unhackable.
Figure~\ref{fig:unhackability_diagram} shows an unhackability diagram including all possible orderings of the three policies $a, b, c$.

\begin{figure*}[ht]
\centering
\includegraphics[width=.9\textwidth]{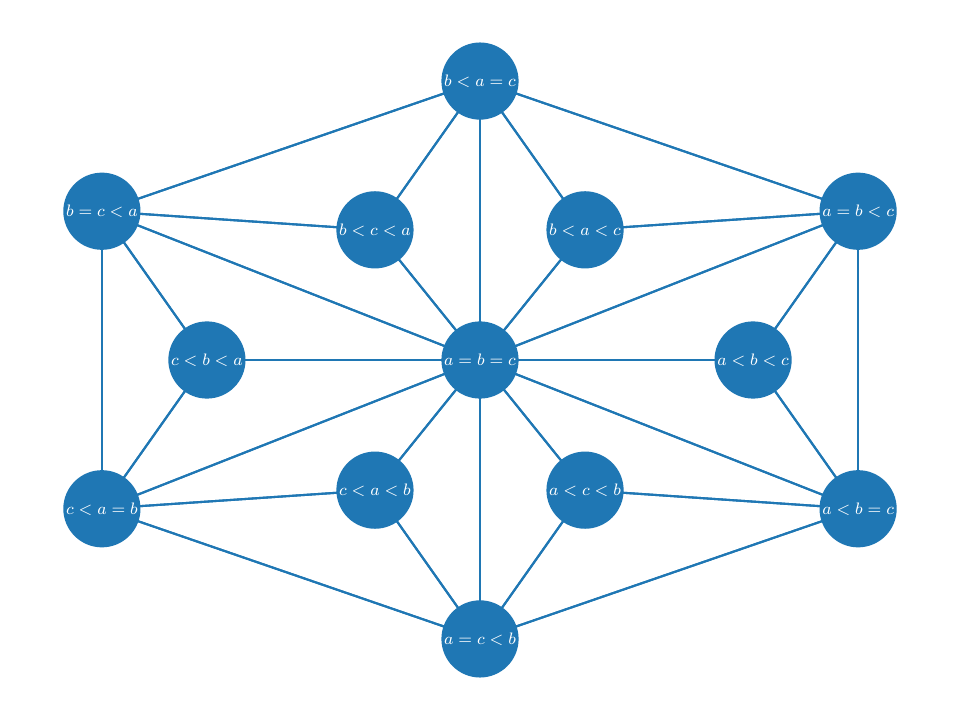}
\caption{Illustration of the unhackable pairs of policy orderings when considering all possible orderings over three policies $a, b, c$. Edges of the graph connect unhackable policy orderings.
}
\label{fig:unhackability_diagram}
\end{figure*}
\newpage

\section{Simplification Diagram}
\label{sec:simplification_diagram}

We can also represent all possible simplifications using a directed graph, which we call a \textbf{simplification diagram}.
This includes a node for every representable ordering and edges pointing from orderings to their simplifications.
Figure~\ref{fig:simplification_diagram} presents a simplification diagram including all possible orderings of three policies $a, b, c$.

\begin{figure*}[ht]
\centering
\includegraphics[width=.9\textwidth]{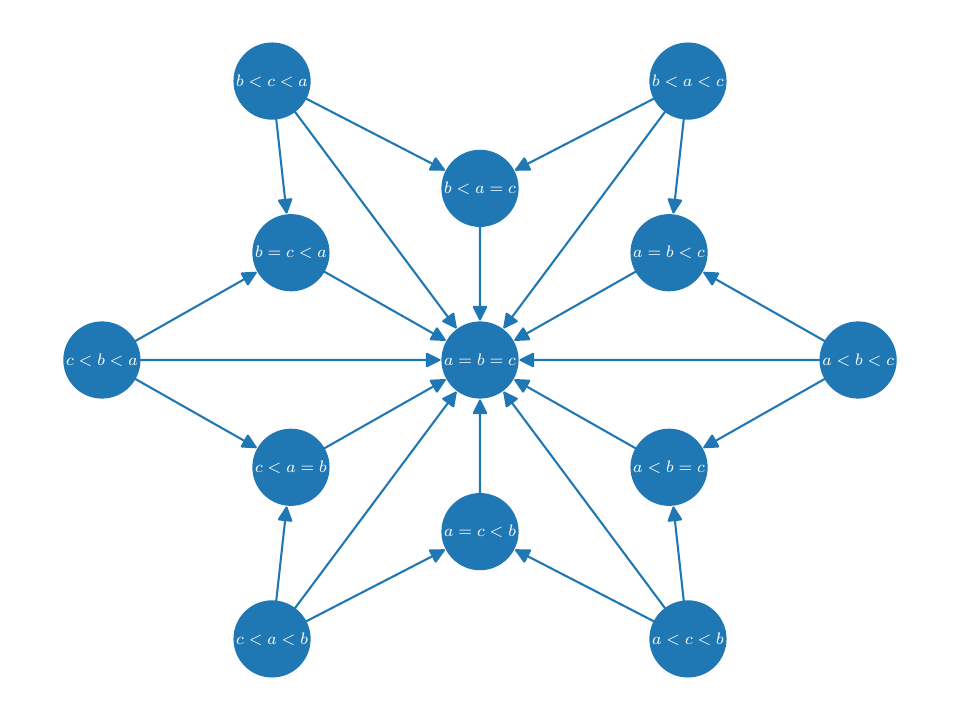}
\caption{Illustration of the simplifications present when considering all possible orderings over three policies $a, b, c$. Arrows represent simplification: the policy ordering at the head of an arrow is a simplification of the policy ordering at the tail of the arrow.}
\label{fig:simplification_diagram}
\end{figure*}

We note that the simplification graph is a subgraph of the unhackability graph.
This will always be the case, since simplification can never lead to reward hacking.

\end{document}